\documentclass[journal]{IEEEtran}
\ifCLASSINFOpdf
\else
\fi
\usepackage[T1]{fontenc}
\usepackage{booktabs}
\usepackage{array}
\usepackage{multirow}
\usepackage{makecell}
\usepackage{tabularx}
\newcolumntype{C}[1]{>{\centering}p{#1}}
\usepackage{cite}
\usepackage{amsmath,amssymb,amsfonts,pifont}
\usepackage{algorithmic}
\usepackage{graphicx}
\usepackage{float}
\usepackage{stfloats}
\usepackage{subfigure}
\usepackage{textcomp}
\usepackage{xcolor}
\usepackage{hyperref}
\usepackage{url}
\usepackage{threeparttable}
\usepackage[linesnumbered,ruled,vlined]{algorithm2e}

\SetCommentSty{mycommfont}
\SetKwInput{KwInput}{Input}                
\SetKwInput{KwOutput}{Output}              

\usepackage[english]{babel}
\usepackage{amsthm}
\newtheorem{theorem}{Theorem}

\newtheorem{remark}{Remark}
\newtheorem{proposition}{Proposition}

\begin{document}
\title{Accelerating AIGC Services with Latent Action Diffusion Scheduling in Edge Networks}

\author{Changfu~Xu,~\IEEEmembership{Member,~IEEE,}
Jianxiong~Guo,~\IEEEmembership{Member,~IEEE,}
Wanyu~Lin,~\IEEEmembership{Member,~IEEE,}
Haodong~Zou,
Yuzhu~Liang,
Xiaowen~Chu,~\IEEEmembership{Fellow,~IEEE,}
Jiannong~Cao,~\IEEEmembership{Fellow,~IEEE,}
and Tian~Wang,~\IEEEmembership{Senior~Member,~IEEE}

\thanks{This work was supported in part by the Joint Funds of the National Natural Science Foundation of China under Grant U25A20436, in part by the National Natural Science Foundation of China (NSFC) under Grant 62372047 and Grant 62602057, in part by Guangdong Science and Technology Program under Grant 2025B0101120006, in part by the Natural Science Foundation of Guangdong Province under Grant 2024A1515011323, in part by the Natural Science Foundation of Jiangxi Province under Grant 20262BAC200420, and in part by Ganpo Talents Support Program under Grant gpjc20260082. (\textit{Corresponding Author: Yuzhu Liang})}
\thanks{Changfu Xu is with the School of Software and IoT Engineering, Jiangxi University of Finance and Economics, Nanchang 330013, China. (E-mail: xuchangfu@jxufe.edu.cn)}
\thanks{Jianxiong Guo and Yuzhu Liang are with the Institute of Artificial Intelligence and Future Networks, Beijing Normal University, Zhuhai 519087, China. Jianxiong Guo is also and the Guangdong Key Lab of AI and Multi-Modal Data Processing, BNBU, Zhuhai 519087, China. (E-mail: jianxiongguo@bnu.edu.cn, cs\_yuzhuliang@163.com)}
\thanks{Wanyu Lin and Jiannong Cao are with the Department of Computing, Hong Kong Polytechnic University, Hong Kong. (E-mail: \{wanylin, csjcao@comp.polyu.edu.hk)}
\thanks{Haodong Zou is with the School of Computer Science and Technology, Anhui University, Hefei 230601, China. (E-mail: zou\_hd@163.com)}
\thanks{Xiaowen Chu is with the Hong Kong University of Science and Technology (Guangzhou), Guangzhou, China. (E-mail: xwchu@hkust-gz.edu.cn)}
\thanks{Tian Wang is with the Institute of Artificial Intelligence and Future Networks, Beijing Normal University, Zhuhai 519087, China, also with the College of Computer and Data Science, Fuzhou University, Fuzhou 350108, China, and also with the School of Information Engineering, Changsha Medical University, Changsha 410219, China. (E-mail: tianwang@bnu.edu.cn)}
}

\markboth{IEEE Transactions on Networking}%
{Shell \MakeLowercase{\textit{et al.}}: Bare Demo of IEEEtran.cls for IEEE Journals}

\maketitle

\begin{abstract}
Artificial Intelligence Generated Content (AIGC) has gained significant popularity for creating diverse content. However, current AIGC models primarily focus on content quality within a centralized framework, resulting in a high service delay and negative user experiences. The workload of an AIGC task is determined by model complexity and multi-layered encoder architectures, posing significant challenges for its modeling, deployment, and scheduling within resource-constrained edge networks. We model an offloading problem among edges for providing real AIGC services and propose \textbf{LAD-TS}, a novel Latent Action Diffusion-based Task Scheduling method that orchestrates multiple edge servers for expedited AIGC services by leveraging the diffusion model's conditional guidance capability and the reinforcement learning's environment interaction ability. We first design an innovative Latent Action Diffusion Network (LADN) model based on traditional diffusion models. We then incorporate the LADN model into deep reinforcement learning, forming an effective LAD-TS method for task scheduling in distributed edge networks. Meanwhile, we design a latent action diffusion strategy to achieve near-optimal decisions quickly by utilizing historical action probabilities instead of Gaussian noise. Furthermore, we implement a refined AIGC model and develop \textbf{DEdgeAI}, a prototype edge system with the refined AIGC model deployment to evaluate our LAD-TS method. DEdgeAI provides a real AIGC service for users, demonstrating up to $7.36\%-18.11\%$ and $34.07\%-94.96\%$ shorter service delays than the state-of-the-art methods and representative commercial platforms, respectively. We release our open-source code at \url{https://github.com/ChangfuXu/DEdgeAI/}.
\end{abstract}


\section{Introduction}
\IEEEPARstart{T}{he} technique of Artificial Intelligence Generated Content (AIGC) has garnered substantial attention and demonstrated considerable success in today’s industrial applications \cite{cao2023comprehensive}. The main aim of AIGC is to autonomously generate the corresponding human-aware content, such as text, images, and video, according to the instructions given by users. Due to the superior capacity of these content creations, the development of AIGC has become a prominent research hotspot \cite{wang2024next}.

In recent years, with the rise of ChatGPT \cite{ouyang2022training}, various AIGC models (such as Midjourney \cite{midjourney} and Sora \cite{liu2024sora}) are constantly being proposed \cite{lin2024blockchain}. For instance, by using Deep Reinforcement Learning (DRL) from human feedback, ChatGPT would generate the most appropriate response for a given instruction, improving the reliability and accuracy of its model \cite{ouyang2022training}. However, current AIGC models focus on the quality performance of content generation and are predominantly operated by the Cloud Server (CS) within a centralized service framework, as shown in Fig.~\ref{Fig1}. This makes a high service response time for AIGC services, leading to a bad Quality of Experience (QoE) \cite{wang2024next}. For example, users generally experience $40-80$ seconds wait for image generation on platforms such as Midjourney \cite{midjourney} and Hugging Face \cite{huggingface}. Therefore, there is a pressing need for the QoE improvement of AIGC applications.

\begin{figure}[!t]
\centering
\includegraphics[width=0.99\linewidth]{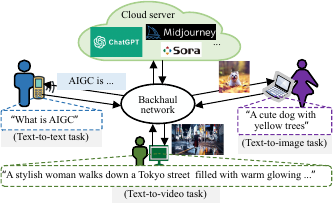}
\caption{An illustration of existing AIGC applications. Users send their prompts to the cloud server that has deployed the corresponding AIGC models by backhaul network. Then, the corresponding content (e.g., image) are generated according to the prompts and are sent back to users.}
\label{Fig1}
\end{figure}

Fortunately, edge computing has emerged as a promising technology to provide lower service delay at Edge Servers (ESs) compared to traditional centralized cloud computing \cite{zhang2024incentive, ye2024asteroid, wang2024invar, li2024online, han2022edgetuner, wang2023edge}. For instance, Xu \textit{et al.} \cite{xu2024dynamic} propose a dynamic parallel multi-ES selection and workload allocation approach in collaborative edge computing systems. In these methods, some task workloads are locally processed at the ES, and some are offloaded to other ESs or the CS for processing \cite{fan2024collaborative, he2024online, khochare2023improved, zeng2023improved}. As a result, the QoE for task requests in both ES and CS is improved. This motivates us to regard ESs as a resource pool to collaboratively process AIGC requests, thus improving AIGC applications' QoE.

However, we observe some challenges to realizing AIGC services in edge networks. \textbf{First Challenge:} Different from traditional tasks, the workload of the AIGC task is mainly determined by the complexity of the AIGC model, rather than by the size of the input data. Moreover, different AIGC models and ESs usually have strong heterogeneity. The service delay is cumulative across multiple links, requiring robust and online decision-making. All these pose a challenge in modeling for the AIGC task offloading. \textbf{Second Challenge:} AIGC tasks are also computation-intensive due to the large parameters of the AIGC model. Thus, when many AIGC tasks simultaneously arrive in the edge system, we have to schedule these tasks to other available ESs for parallel processing due to limited computing resources on the local ES. However, task scheduling is an online decision-making problem, making it computationally infeasible to attain an optimal solution in polynomial time. In addition, although many DRL-based scheduling solutions have recently been proposed to reduce service delay, they still face several challenges, such as limited exploration, sample inefficiency, and performance instability, which can lead to degraded user QoE. \textbf{Third Challenge:} AIGC services are usually large models and use several text encoders. For example, a classic AIGC model DALL$\cdot$E \cite{dalle} has 12 billion parameters. Another well-known AIGC model, Stable Diffusion 3 (SD3) \cite{ho2020denoising} has three pretrained text encoders of OpenCLIP-ViT/G, CLIP-ViT/L, and T5xxl to encode text representations and an improved autoencoding model to encode image tokens. This makes the deployment of AIGC services require a large amount of memory. However, edge devices (e.g., Nvidia Jetson) have limited memory resources, which leads to a significant challenge in the real deployment of AIGC services on edge networks.

To address the above challenges, we propose \textbf{LAD-TS}, a novel Latent Action Diffusion-based Task Scheduling method that can orchestrate multiple ESs for expedited AIGC services. Moreover, we develop \textbf{DEdgeAI}, a prototype edge system to implement and evaluate our LAD-TS method for real AIGC services. Specifically, \textit{in response to the first challenge}, by analyzing the architectural features of an AIGC model, we model the AIGC task offloading problem in heterogeneous distributed edge environments for providing real AIGC services, and formulate it as an online Integer Non-Linear Programming (INLP) problem with NP-hardness. The problem objective is designed to minimize the average delay of all task offloading in the system, ensuring the QoE of AIGC applications in distributed edge networks. \textit{For the second challenge}, we first observe that the diffusion models \cite{ho2020denoising, esser2024scaling} have significant advantages: 1) \textbf{Improving Exploration:} They generate diverse, high-quality actions, leading to more efficient exploration of the state space compared to traditional DRL. 2) \textbf{Increased Sample Efficiency:} By synthesizing realistic trajectories, they maximize the utility of limited training data. 3) \textbf{Stabilized Learning:} Their structured denoising process refines actions gradually, resulting in smoother policy updates and greater training stability compared to the abrupt changes common in traditional DRL. We then design an innovative Latent Action Diffusion Network (LADN) model based on the traditional diffusion model. The LADN model is incorporated into the DRL with a decoupled offline-online design, forming an effective LAD-TS method to generate the approximate optimal offloading decisions by utilizing the diffusion model's condition generation capability and the DRL's environment interaction ability. As a result, the LAD-TS can minimize service delays under multiple resource constraints. We also design a latent action diffusion strategy that replaces random Gaussian noise with historical action probabilities as a more informative prior to guide decision generation, which enables fast achievement of near-optimal decisions. Moreover, we provide a theoretical analysis to show the LAD-TS's probability derivation. \textit{To address the third challenge}, we first design a refined SD3 medium model, named \textbf{reSD3-m}, by removing T5xxl encoder. Afterwards, we develop the DEdgeAI system with the reSD3-m deployment in NVIDIA Jetson devices. Furthermore, we implement the LAD-TS method in the DEdgeAI system to provide a real AIGC service at the edge. Finally, extensive evaluations and test-bed results show the effectiveness of our LAD-TS method and DEdgeAI system while reducing the memory occupation on the ESs significantly. This paper makes the following main contributions:
\begin{itemize}
    \item \textbf{Modeling:} Unlike conventional task scheduling formulations in edge computing, we use a diffusion-based AIGC model architecture as an example to develop an edge offloading model for providing AIGC services. The resulting optimization problem is formulated as an online INLP problem, aiming to minimize the average service delay of all offloaded tasks in the system. We prove the NP-hardness of the offline counterpart of the problem.
     \item \textbf{Method:} We propose an innovative LAD-TS method for expedited AIGC services with a latent action diffusion-based strategy, achieving near-optimal decisions quickly. Furthermore, we implement our LAD-TS method as an online distributed algorithm with corresponding theoretical analysis, showing the LAD-TS's probability derivation. Extensive simulation evaluations reveal that our LAD-TS method significantly reduces the service delay by $8.58\%$ to $33.67\%$ and the training episode duration by at least $60\%$ compared to the state-of-the-art methods.
    \item \textbf{System:} We design the reSD3-m model and develop the DEdgeAI system with the reSD3-m deployment to implement and evaluate our LAD-TS method. Test-bed results demonstrate that the LAD-TS achieves a service delay improvement by $7.36\%$ to $11.87\%$ compared to the state-of-the-art methods. The reSD3-m reduces the memory occupation by about 60\% compared to the original SD3 medium. The DEdgeAI reduces the AIGC service delays by 34.07\% to 94.96\%, compared to existing representative platforms.
\end{itemize}

The remainder of this paper is organized as follows. Section \ref{sec2} introduces the related work. Section \ref{sec3} presents the system model and problem formulation. Section \ref{sec4} and Section \ref{sec5} give the specific details and effectiveness evaluation of our method, respectively. The test-bed results and conclusion are shown in Section \ref{sec6} and Section \ref{sec7}, respectively.

\section{Related Work}\label{sec2}
In this section, we briefly review the related work of cloud-enabled AIGC, edge-enabled AIGC, and diffusion scheduling.

\textbf{Cloud-enabled AIGC:}
A large amount of research currently focuses on cloud-enabled AIGC with a centralized framework, where AIGC services are primarily deployed on the CS for clients' usage \cite{du2023ai, midjourney} (see in Fig. \ref{Fig1}). For instance, in the ChatGPT \cite{van2023chatgpt} application, all users send their prompts to the CS for processing in a centralized way. Similarly, other AIGC applications, such as SD3 models \cite{sd3medium} and VideoGPT \cite{maaz2023video}, for generating images also use the resource in the CS to process all user requests. These applications have been widely studied due to their superior generation capabilities. However, they just adopt a centralized platform that has limited network bandwidth. Additionally, they also demand a huge computing capacity for the CS. Consequently, the response time for achieving processing results is usually high due to the long distance between users and the CS, especially for many user requests. 

\begin{table*}[!t]
    \caption{Comparison of ours and current representative methods}
    \centering
    \resizebox{\textwidth}{!}{
    \begin{tabular}{l|ccccc}
    \toprule
    Method &Cloud-enabled AIGC &Edge-enabled AIGC & Diffusion scheduling &Latent action diffusion &Offline-online decoupling\\
    \midrule
    Midjourney \cite{midjourney} &\checkmark &\ding{53} &\ding{53} &\ding{53} &\ding{53}\\
    Xu \textit{et al.} \cite{xu2025enhance}, Wu \textit{et al.} \cite{wu2025qoe} &\ding{53}  &\checkmark &\ding{53} &\ding{53} &\ding{53}\\
    Du \textit{et al.} \cite{du2024diffusion}, Tang \textit{et al.} \cite{xu2026eat} &\ding{53}  &\checkmark &\checkmark &\ding{53} &\ding{53}\\
    Ours &\ding{53} &\checkmark &\checkmark &\checkmark &\checkmark\\
    \bottomrule
    \end{tabular}}
    \label{table1}
\end{table*}

\textbf{Edge-enabled AIGC:}
Due to the advantage of low service delay in edge computing, the edge-enabled AIGC methods \cite{yang2026online,wang2024next, xu2023sparks} have been proposed to provide AIGC services at the network edge. Since the DRL can optimize offloading decisions through interaction with the environment \cite{tang2023multi, tang2022deep}, most edge-enabled AIGC methods use DRL (e.g., \cite{mnih2015human, haarnoja2018soft}) techniques to realize task scheduling. For example, Xu \textit{et al.} \cite{xu2025enhance} propose an adaptive multi-server collaboration approach for edge-enabled AIGC based on DRL, thereby enhancing AIGC efficiency. Additionally, Wu \textit{et al.} \cite{wu2025qoe} introduce DRL-based task offloading and resource allocation for edge-enabled AIGC services based on DRL. These DRL-based methods can efficiently adapt to unforeseen changes in task characteristics due to their dynamic learning framework, making them highly suitable for scheduling challenges in real-time systems. However, these methods consequently suffer from their inherent limitations of constrained exploration, poor sample efficiency, and performance instability \cite{luong2019applications}.

\textbf{Diffusion Scheduling:}
Diffusion models have recently gained popularity in machine learning, particularly for applications such as image and video generation and molecular design \cite{croitoru2023diffusion}. In computer vision, diffusion-based neural networks generate images by reversing the diffusion process, which typically involves introducing Gaussian noise and then gradually removing it \cite{ho2020denoising}. Recently, a combination of diffusion models and DRL has been explored for task scheduling in edge computing \cite{du2024exploring, du2024enhancing}. For example, Du \textit{et al.} \cite{du2024diffusion} propose a Deep Diffusion-based Soft-Actor-Critic Task Scheduling (D2SAC-TS) algorithm that integrates diffusion models with DRL techniques to enhance decision optimization for task scheduling in edge computing systems. Furthermore, Xu \textit{et al.} \cite{xu2026eat} present a QoS-aware edge-collaborative AIGC task scheduling algorithm with attention-guided diffusion DRL. Although these methods optimize the decision policy by the combination of the diffusion model and the DRL technique, they fail to find the scheduling decision quickly and exhibit performance instability in complex edge environments. 

\textbf{Comparison Analysis:}
Table \ref{table1} summarizes the difference between current representative methods and our method. \textit{First}, current AIGC methods, such as \cite{midjourney, dalle}, are mainly enabled by the CS with a centralized platform. Consequently, our method is enabled by ESs with a distributed AIGC service deployment and task scheduling, alleviating network bandwidth limitations. \textit{Second}, while prior work (e.g., \cite{xu2025enhance, wu2025qoe}) has explored distributed AIGC service deployment and scheduling, these approaches are usually grounded in traditional DRL. \textit{Third}, one of the most related works is literature \cite{du2024diffusion} which considers the diffusion scheduling. However, this method learns a decision policy from Gaussian noise. In contrast, our method generates better decision policies from accumulated policy priors by leveraging historical action probabilities. Additionally, we propose a novel dual LADN architecture with a decoupled offline-online design. Moreover, our method features provable probabilistic derivation and linear time complexity, enabling near-optimal decisions of AIGC task scheduling in distributed edge networks. Also, we design a refined AIGC model and develop a prototype system to implement our method for providing a real AIGC service.

\section{System Model and Problem Formulation}\label{sec3}
In this section, we first present the AIGC system model and then formally formulate our problem. 
Table \ref{symbol-table} summarizes the notations frequently used in this paper.

\begin{table}[!t]
    \caption{Symbols frequently used in this paper.}
    \centering
    \begin{tabular}{p{30pt}<{\centering}|p{190pt}}
    \toprule
    Symbol & Description\\
    \midrule
    $\mathcal{T}$ & The set of time slots.\\
    $\mathcal{B}$ & The set of ESs or BSs.\\
    $\mathbb{N}_{b,t}$ & The task set that arrives to BS $b\in\mathcal{B}$ at time slot $t\in \mathcal{T}$.\\
    $\Delta$ & The length of each time slot $t \in \mathcal{T}$.\\
    $d_{n}$ & The data size of task $n \in \mathbb{N}_{b,t}$.\\
    $z_{n}$ & The generation quality demand of an AIGC task $n$.\\
    $\rho_{n}$ & The required computation density of offloading task $n$. \\
    $\boldsymbol{\psi}$ & The decision variable of task scheduling decision.\\
    $\boldsymbol{v}$ & The transmission rate between BSs\\
    $f_{b'}$ & The CPU or GPU computing capacity of ES $b' \in \mathcal{B}$ \\
    $T_{b,n,t,b'}^{\text{serv}}$ & The service delay of offloading task $n$ that is allocated from BS $b$ to ES $b'$ at time slot $t$.\\
    $T_{b,n,t,b'}^{\text{wait}}$ & The waiting time of task $n$ in the transmission queue at local BS $b$ and the processing queue at ES $b'$.\\
    $q_{t,b'}$ & The processing queue workload length of ES $b'$ at the end of time slot $t-1.$\\
    $q^{\text{bef}}_{n,t,b'}$ & The workload length of the processing queue at ES $b'$ before receiving task $n$ at time slot $t$.\\
    \bottomrule
    \end{tabular}
    \label{symbol-table}
\end{table}

\subsection{AIGC System Model}
We consider AIGC applications within a distributed edge system as illustrated in Fig. \ref{Fig2}. The network consists of multiple Base Stations (BSs), each equipped with an ES. These ESs have different computing capacities and deploy different AIGC services, such as SD \cite{ho2020denoising} and LLaMA \cite{touvron2023llama} models. All BSs are connected via a wired core network. Each ES has a scheduler with a processing queue. The processing queue retains the tasks pending computation by the available CPU and GPU resources. Then, when AIGC tasks reach a BS within a given time slot, the BS's scheduler determines the allocation of task workloads to the appropriate ESs for parallel processing. Let $\mathcal{B} =\{1,\cdots, B\}$ represent the set of ESs or BSs, with $b \in \mathcal{B}$ indicating the $b$-th ES and $B$ denoting the total count of ESs. The system operates over consecutive time slots, denoted by $\mathcal{T} =\{1,\cdots, |\mathcal{T}|\}$ \cite{farhadi2021service, guo2024betweenness}, where $t \in \mathcal{T}$ signifies the $t$-th time slot. We set the duration of each time slot $t \in \mathcal{T}$ to $\Delta$ seconds. In subsequent sections, we will elaborate on the AIGC task model, scheduling decision variable, and service delay model.

\subsubsection{AIGC Task Model}
We focus on the text-to-image and image-to-image tasks in AIGC applications. For convenience, we use a diffusion-based AIGC service as an example for modeling. Notably, in the text-to-image task, each task's input is a piece of text and its corresponding demand of image quality. In contrast, the input of each image-to-image task has an image and the prompt of text description. Unlike general tasks, the workload of an AIGC task is typically not determined by its data size. To address this issue, we first let $\mathbb{N}_{b,t} = \{1, 2, \cdots, N_{b,t}\}$ denote a set of AIGC tasks arrived to BS $b$ at the same time slot $t$, where $n \in \mathbb{N}_{b,t}$ denotes as the $n$-th task. Then, for a specific text-to-image or image-to-image task $n$, we use a variable $d_{n}$ (in bits) to represent the inputted data (text, or text and image) size of task $n$. Furthermore, by analyzing the network structure of the SD model, we observe that the generating image quality relies on the number of denoising steps. In other words, the more denoising steps the better image generation. Thus, the number of denoising steps can well represent the demand for generating image quality in the SD model. Based on this observation, we define a variable $z_{n}$ for the number of denoising steps and then the workload of the task $n$ can be modeled as $\rho_{n}\cdot z_{n}$ (in CPU or GPU cycles if GPU is available). Here, $\rho_{n}$ (in CPU cycles/step) is the required computing resources of each denoising step for task $n$ and is determined by the task type, AIGC model types, and ES performance. 

Notably, our modeling method can be extended to other AIGC architectures' task offloading by changing partial formulation expressions according to specific execution characteristics. For instance, for a Large Language Model (LLM) task and a multimodal task, we can redesign the workload representation $\rho_{n} \cdot z_{n}$ by $c_{n}^{\text{pre}} \cdot L_{n}^{\text{in}}$ + $c_{n}^{\text{dec}} \cdot L_{n}^{\text{out}}$ and $c_{n}^{\text{vis}}\cdot N_{n}^{\text{vis}}$ + $c_{n}^{\text{pre}} \cdot L_{n}^{\text{in}}$ + $c_{n}^{\text{dec}} \cdot L_{n}^{\text{out}}$, respectively. Here, the $c_{n}^{\text{pre}}$, $c_{n}^{\text{dec}}$, and $N_{n}^{\text{vis}}$ are the unit computation overhead measured through offline profiling. The $L_{n}^{\text{in}}$ is the number of input tokens. The $L_{n}^{\text{out}}$ is the maximum number of output tokens. The $N_{n}^{\text{vis}}$ is the number of visual tokens. Furthermore, the following formulations can remain the same for CPU or GPU workloads.

\subsubsection{Scheduling Decision Variable}
Tasks originate from different BSs and are scheduled by a controller at the local BS in a distributed way. Upon task arrival at a BS, the local scheduler evaluates the current resource states of all ESs (e.g., queue lengths, computing capacities) and allocates each task to a suitable ES to generate an image for each time slot in our problem. Let a binary variable $\psi_{b,n,t,b'} \in \{0, 1\}$ denote the allocation decision of offloading task $n$ from the BS $b$ to the ES $b' \in \mathcal{B}$ in time slot $t$. This decision is encoded by the binary variable $\psi_{b,n,t,b'}$, ensuring a feasible constraint that each task is executed by exactly one ES, and that ESs across the network process tasks in parallel under a coordinated scheduling strategy. The feasible constraint is presented as follows:
\begin{equation}\label{Eq-constrain}
    \sum\nolimits_{b'\in \mathcal{B}}\psi_{b,n,t,b'}=1,\ \forall b \in \mathcal{B}, n \in \mathbb{N}_{b,t}, t \in \mathcal{T}.
\end{equation}

\begin{figure}[!t]
    \centering
    \includegraphics[width=0.9\linewidth]{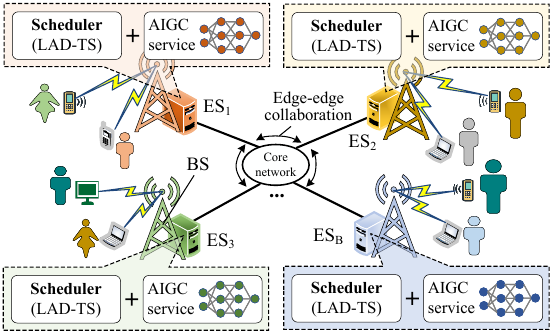}
    \caption{An illustration of the efficient distributed AIGC in this paper. For each AIGC task in each time slot, the scheduler will decide which ES to process the task by the proposed LAD-TS method.}
    \label{Fig2}
\end{figure}

\subsubsection{Service Delay Model}
AIGC tasks are processed in a first-come-first-served way. That is to say, each task should be processed after the completion of its previous arrival task in the processing queue. Let $T_{b,n,t,b'}^{\text{serv}}$ denote the total service delay of task $n$ that is uploaded to BS $b$ and is offloaded to ES $b'$ for processing in the time slot $t$. Thus, the $T_{b,n,t,b'}^{\text{serv}}$ includes the task transmission delays from the user end to BS $b$ and from BS $b$ to BS $b'$, the task computing delay on ES $b'$, the task waiting time $T^{\text{wait}}_{b,n,t,b'}$ in the processing queue at ES $b'$, and the transmission delay of generated result from BS $b'$ to BS $b$ and from BS $b$ to the user end. Thus, $T_{b,n,t,b'}^{\text{serv}}$ can be formulated as 
\begin{equation}\label{Eq-process-delay}
    T_{b,n,t,b'}^{\text{serv}} = \psi_{b,n,t,b'} \cdot(\frac{d_{n}}{v_{n,b',t}} + \frac{\rho_{n} \cdot z_{n}}{f_{b'}} +T^{\text{wait}}_{b,n,t,b'} + \frac{\tilde{d}_{n}}{v_{b',n,t}}),
\end{equation}
where $v_{n,b',t}$ (in bits/s) is the transmission rate from user end to BS $b'$, $v_{b',n,t}$ (in bits/s) is the transmission rate from BS $b'$ to BS $n$, $f_{b'}$ (in CPU or GPU cycles/s) is the CPU or GPU computing capacity of ES $b'$ if GPU is available, and $\tilde{d}_{n}$ is the result (i.e., image) size of processing task $n$. Thus, $T^{\text{wait}}_{b,n,t,b'}$ can be achieved by 
\begin{equation}\label{waiting-time}
	 T^{\text{wait}}_{b,n,t,b'} = ({q_{t-1,b'}+q^{\text{bef}}_{n,t,b'}})/{f_{b'}},
\end{equation}
where $q^{\text{bef}}_{n,t,b'}$ (in CPU or GPU cycles) is the workload length of the processing queue at ES $b'$ before receiving task $n$, which can be achieved by system observation. $q_{t-1,b'}$ (in CPU or GPU cycles) is the workload length of the processing queue of ES $b'$ at the end of time slot $t-1$, which can be updated by $q_{t,b'} = $
\begin{equation}\label{queue-update}
    \max\{q_{t-1,b'} + \sum_{b\in\mathcal{B}} \sum_{n \in \mathbb{N}_{b,t}} \psi_{b,n,t,b'} \cdot \rho_{n} \cdot z_{n} - f_{b'} \Delta,\; 0\},
\end{equation}
where the task workload of each ES at initial time slot $0$ is set to $0$, i.e., $q_{0,b'} = 0$ for $\forall b' \in \mathcal{B}$, and the queue workload length before task $1$ in time slot $0$ are set to $0$, i.e., $q^{\text{bef}}_{0,0,b'} = 0$ for $\forall b' \in \mathcal{B}$.

\subsection{Problem Formulation}
The objective of our problem for efficient AIGC is to minimize the average service delay of all the offloading tasks in the whole system, which is formulated as an online INLP optimization problem:
\begin{align}\label{Eq-problem}
    &\min_{\boldsymbol{\psi}}\lim_{|\mathcal{T}| \rightarrow \infty}\frac{1}{|\mathcal{T}|}\sum_{t\in\mathcal{T}} \sum_{b\in\mathcal{B}}\sum_{n\in\mathbb{N}_{b,t}}\sum_{b'\in\mathcal{B}} T_{b,n,t,b'}^{\text{serv}} \\
	&\mbox{s.t.}\quad\text{Eqn.}\; (\ref{Eq-constrain}), \nonumber \\
    &\quad\quad \psi_{b,n,t,b'} \in \{0,1\},\; \forall b,b' \in \mathcal{B}, n \in \mathbb{N}_{b,t}, t \in \mathcal{T}.\nonumber
\end{align}
Here, the optimization variable is task allocation decision $\boldsymbol{\psi} = \{\psi_{b,n,t,b'}\}_{b,b'\in \mathcal{B}, n\in \mathbb{N}_{b,t},t\in \mathcal{T}}$ in the problem. This offline counterpart of this problem is NP-hard, as proved by the following Theorem \ref{Theorem1}.
\begin{theorem}\label{Theorem1}
The offline counterpart of the problem (\ref{Eq-problem}) is an NP-hard problem.
\end{theorem}
\begin{proof}
The offline counterpart of the problem (\ref{Eq-problem}) is proven to be NP-hard through reducing from the multi-knapsack problem \cite{gu2021layer}. Firstly, all the tasks are viewed as candidate objects in the problem (\ref{Eq-problem}). The set $\mathcal{B}$ of ESs is viewed as multi-knapsacks. The tasks offloaded to the ES set are viewed as candidate objects placed into multiple knapsacks under the weight constraints of ES capacities $f_{1}$, $\cdots$, $f_{B}$. Thus, the problem (\ref{Eq-problem}) is a multi-knapsack problem. Secondly, as a general case of the knapsack problem that is well-known NP-hard, the offline counterpart of the problem (\ref{Eq-problem}) is also NP-hard.
\end{proof}

\section{Proposed Method and Algorithm}\label{sec4}
This section presents the specific design of our method with the proof of probability derivation, and gives our algorithm implementation characterized by linear time complexity.

\subsection{Method Design}\label{sec4-1}
Driven by the NP-hardness of our problem and inspired by the exceptional generation capability of the Denoising Diffusion Probabilistic Model (DDPM) \cite{ho2020denoising}, we propose a LAD-TS method to quickly achieve a near-optimal solution to our problem (\ref{Eq-problem}) by incorporating the diffusion model and DRL technique. The overall architecture of our LAD-TS method is implemented with the basic SAC \cite{haarnoja2018soft} model and reverse diffusion processing as shown in Fig.~\ref{Fig3}. Our architecture includes two parts: online task scheduling and offline network training. In particular, two Latent Action Diffusion Networks (LADNs), named s-LADN and t-LADN, are designed using reverse diffusion processing. The s-LADN and t-LADNs are used as the two actors in the online task scheduling and the periodic offline training, respectively. In the task scheduling part, the actor will generate a task scheduling decision by the s-LADN according to the inputs of the system state and latent action probability. The network training consists of the t-LADN (the same network structure and parameter as the s-LADN), two Critic Networks (CNs), and two Target Networks (TNs). The t-LADN is periodically trained using the history batch samples extracted from the experience pool. The CNs and TNs have the same network structure, and generate the system's evaluation Q-values $\boldsymbol{Q}_{\text{eval}}$ and target Q-values $\boldsymbol{Q}_{\text{target}}$, respectively. Notably, two key improvements of our method are designed: (1) a latent action diffusion strategy to find the near-optimal decision quickly and (2) a dual LADN architecture with an offline-online decoupling mechanism to avoid the poor performance of offline training. The state space, action space, reward function, actor structure, latent action diffusion strategy, and network training processing will be described as follows.

\begin{figure}[!t]
    \centering
    \includegraphics[width=\linewidth]{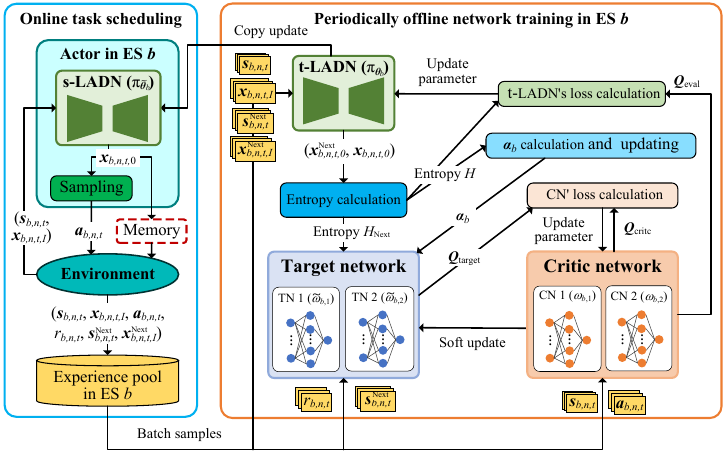}
    \caption{The overall architecture of our method. For each ES $b$, new arrival AIGC tasks are online offloaded to ESs for parallel processing by the actor. The actor utilizes the system state $\boldsymbol{s}_{b,n,t}$ and historical action probability $\boldsymbol{x}_{b,n,t}$ with edge-edge collaboration. The $\boldsymbol{x}_{b,n,t}$is stored in memory. The t-LADN model is periodically offline trained using history samples.}
    \label{Fig3}
\end{figure}

\textbf{State space.} In our method, each task $n \in \mathbb{N}_{b,t}$ is sequentially processed with one-by-one. Let a vector $\boldsymbol{s}_{b,n,t}$ denote the system state when a task $n$ is arrived at the BS $b$ at time slot $t$. Intuitively, the arrival task size, the task workload length, and the processing queue length at each ES will affect the decision generation and service delay of task offloading. Thus, $\boldsymbol{s}_{b,n,t}$ is formulated as
\begin{equation}
	\boldsymbol{s}_{b,n,t} = [d_{n},\; \rho_{n}\cdot z_{n},\; \boldsymbol{q}_{t-1}],
\end{equation}
where $\boldsymbol{q}_{t-1}=[q_{t-1,1}, q_{t-1,2}, \cdots, q_{t-1,B}]$ is achieved by the Eqn. (\ref{queue-update}). Furthermore, let a variable $\boldsymbol{s}_{b,n,t}^{\text{next}}$ denote the next system state after offloading a task $n$. $\boldsymbol{s}_{b,n,t}^{\text{next}}$ is formulated as
\begin{equation}
	\boldsymbol{s}_{b,n,t}^{\text{next}} = \begin{cases}
        \boldsymbol{s}_{b,n+1,t}, & n < N_{b,t}.\\
        \boldsymbol{s}_{b,1,t+1}, & \text{Otherwise.}
        \end{cases}
\end{equation}

\textbf{Action space.} We use a binary vector $\boldsymbol{a}_{b,n,t}$ = $[\psi_{b,n,t,1}$, $\psi_{b,n,t,2}$, $\cdots$, $\psi_{b,n,t,B}]$ to represent the action of a task $n$ that is arrived to the BS $b$ at time slot $t$. Let $\mathcal{A}$ represent the set of all action spaces. We have $\mathcal{A} = \{[1, 0, 0, \cdots, 0] $, $[0, 1, 0, \cdots, 0]$, $\cdots$, $[0, 0, 0, \cdots, 1]\}$. In our problem, an action $ \boldsymbol{a} \in \mathcal{A}$ is determined by our t-LADN, i.e., $\pi_{\boldsymbol{\theta}_{b}}(\cdot)$, with the inputs of current system state $\boldsymbol{s}_{b,n,t}$, latent action probability $\boldsymbol{x}_{b,n,t,I}$, and timestep $I$. Furthermore, $\boldsymbol{a}_{b,n,t}$ is formulated as
\begin{equation}\label{action_eq}
    \boldsymbol{a}_{b,n,t} = \mathop{\arg\max}\limits_{\boldsymbol{a}}\{ \pi_{\boldsymbol{\theta}_{b}}(\boldsymbol{a}|\boldsymbol{s}_{b,n,t}, \boldsymbol{x}_{b,n,t,I},I), \forall{\boldsymbol{a}} \in \mathcal{A}\},
\end{equation}
where $\pi_{\boldsymbol{\theta}_{b}}(\boldsymbol{a}|\boldsymbol{s}_{b,n,t}, \boldsymbol{x}_{b,n,t,I},I)$ is the probability of generating the action $\boldsymbol{a}$. $\boldsymbol{\theta}_{b}$ is the t-LADN's parameters in ES $b$. 

\textbf{Reward function.} We hope the fewer service delays the more rewards. Thus, the negative task service delay is defined as the reward in our method, i.e., the reward $r_{b,n,t}$ of offloading task $n$ under giving $\boldsymbol{a}_{b,n,t}$ is formulated as
\begin{equation}\label{Eq-reward}
    r_{b,n,t} = -T_{b,n,t,b'}^{\text{serv}}.
\end{equation}

\begin{figure*}[!t]
    \centering
    \includegraphics[width=\linewidth]{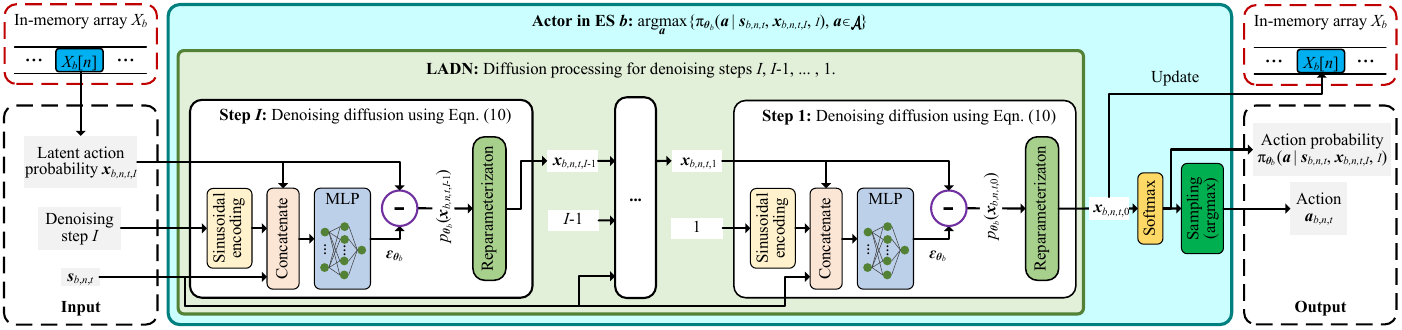}
    \caption{The Actor structure with proposed LADN model. The actor input are the timestep $I$, latent action probability $\boldsymbol{x}_{b,n,t,I}$, and system state $\boldsymbol{s}_{b,n,t}$. The output is the action decision $\boldsymbol{a}_{b,n,t}$. The historical action probability $\boldsymbol{x}_{b,n,t,0}$ is stored (or updated) into the array $X_{b}[n]$.}
    \label{Fig4}
\end{figure*}

\textbf{Actor structure.} The actor structure consists of a LADN, a softmax unit, and a sampling unit as shown in Fig. \ref{Fig4}. The LADN model has $I$ denoising steps. Each denoising step has a sinusoidal encoding unit, a concatenate unit, a Multi-Layer Perception (MLP), and a reparameterization unit. At the first denoising step, the input is the $\boldsymbol{x}_{b,n,t,I}$, $I$, and $\boldsymbol{s}_{b,n,t}$. The output is $\boldsymbol{x}_{b,n,t,I-1}$. Then, the $\boldsymbol{x}_{b,n,t,I-1}$, $I-1$, and $\boldsymbol{s}_{b,n,t}$ are used as the input of the second denoising step. After, followed by the operation of $I$ denoising steps, the action probability $\boldsymbol{x}_{b,n,t,0}$ will be outputted in the last denoising step. Furthermore, the action probability $ \pi_{\boldsymbol{\theta}_{b}}(\boldsymbol{a}|\boldsymbol{s}_{b,n,t}, \boldsymbol{x}_{b,n,t,I},I)$ is achieved by taking the softmax operation on $\boldsymbol{x}_{b,n,t,0}$. Finally, the optimal $\boldsymbol{a}_{b,n,t}$ is achieved by the sampling operation using Eqn. (\ref{action_eq}). Here, the $\boldsymbol{x}_{b,n,t,0}$ is achieved by the \textbf{probability derivation} as shown in the following Theorem \ref{Theorem2}.

\begin{theorem}\label{Theorem2}
The action probability $\boldsymbol{x}_{b,n,t,0}$ is inferred by the update rule: $\boldsymbol{x}_{b,n,t,i-1} =$
\begin{equation}\label{Eq-update-rule}
    \frac{1}{\sqrt{\lambda_{i}}}(\boldsymbol{x}_{b,n,t,i}  - \frac{\beta_{i}}{\sqrt{1-\overline{\lambda}}}\boldsymbol{\epsilon}_{\boldsymbol{\theta}_{b}}(\boldsymbol{x}_{b,n,t,i}, i, \boldsymbol{s}_{b,n,t}))+ \frac{\tilde{\beta}_{i}}{2} \cdot \boldsymbol{\epsilon},
\end{equation}
where $i=I, I-1, \cdots, 1$, $\beta_{i} = 1 - e^{-\frac{\beta_{\text{min}}}{I}-\frac{2i-1}{2I^{2}}(\beta_{\text{max}}-\beta_{\text{min}})}$ represents the forward process variance \cite{ho2020denoising}, $\tilde{\beta}_{i} = \frac{1-\overline{\lambda}_{i-1}}{1-\overline{\lambda}_{i}}\cdot\beta_{i}$ is a deterministic variance parameter, $\lambda_{i} = 1 - \beta_{i}$, $\overline{\lambda}_{i} = \prod_{m=1}^{i}\lambda_{m}$ is the cumulative production of $\lambda_{m}$ over previous $m\;(m \leq i)$ denoising steps, $\lambda_{m} = 1 - \beta_{m}$, $\boldsymbol{\epsilon}_{\boldsymbol{\theta}_{b}}(\boldsymbol{x}_{b,n,t,i}, i, \boldsymbol{s}_{b,n,t})$ is the MLP's output, and $\boldsymbol{\epsilon}\sim \mathcal{N}(\textbf{0}, \textbf{I})$.
\end{theorem}
\begin{proof}
Firstly, according to the forward diffusion processing in \cite{ho2020denoising}, we have
\begin{align}\label{Eq-forward-diffusion}
    \boldsymbol{x}_{b,n,t,i} = \sqrt{\overline{\lambda}_{i}}\cdot\boldsymbol{x}_{b,n,t,0} + \sqrt{1 - \overline{\lambda}_{i}}\cdot\boldsymbol{\epsilon}.
\end{align}

Secondly, let $p(\boldsymbol{x}_{b,n,t,i-1}|\boldsymbol{s}_{b,n,t}, i, \boldsymbol{x}_{b,n,t,i})$ denote the transition function from $\boldsymbol{x}_{b,n,t,i}$ to $\boldsymbol{x}_{b,n,t,i-1}$ and it follows a Gaussian distribution given by $p_{\boldsymbol{\theta}_{b}}(\boldsymbol{x}_{b,n,t,i-1}|\boldsymbol{s}_{b,n,t}, i, \boldsymbol{x}_{b,n,t,i})=$
\begin{equation}\label{Eq-diffusion-distribution}
\mathcal{N}(\boldsymbol{x}_{b,n,t,i-1};\boldsymbol{\mu}_{\boldsymbol{\theta}_{b}}(\boldsymbol{x}_{b,n,t,i},i,\boldsymbol{s}_{b,n,t}), \tilde{\beta}_{i}\boldsymbol{\text{I}}), 
\end{equation}
where $i\in\{1, 2 $, $\cdots, I\}$, $\boldsymbol{\mu}_{\boldsymbol{\theta}_{b}}(\boldsymbol{x}_{b,n,t,i},i,\boldsymbol{s}_{b,n,t})$ is the distribution mean and can be learned by a deep model, $\tilde{\beta}_{i} = \frac{1-\overline{\lambda}_{i-1}}{1-\overline{\lambda}_{i}}\cdot\beta_{i}$ is a deterministic variance parameter, $\beta_{i} = 1 - e^{-\frac{\beta_{\text{min}}}{I}-\frac{2i-1}{2I^{2}}(\beta_{\text{max}}-\beta_{\text{min}})}$ represents the forward process variance controlled by the variational posterior scheduler \cite{ho2020denoising}, $\tilde{\lambda}_{i} = \prod_{m=1}^{i}\lambda_{m}$ is the cumulative production of $\lambda_{m}$ over previous $m\;(m \leq i)$ denoising steps, and the $\lambda_{m} = 1 - \beta_{m}$.

Thirdly, by applying the Bayesian formula and combining the reconstructed sample $\boldsymbol{x}_{b,n,t,0}$ in Eqn. (\ref{Eq-forward-diffusion}), we have $\boldsymbol{\mu}_{\boldsymbol{\theta}_{b}}(\boldsymbol{x}_{b,n,t,i},i,\boldsymbol{s}_{b,n,t})=$\\
\begin{equation}
\frac{1}{\sqrt{\lambda_{i}}}(\boldsymbol{x}_{b,n,t,i}  - \frac{\beta_{i}}{\sqrt{1-\overline{\lambda}}}\cdot\boldsymbol{\epsilon}_{\boldsymbol{\theta}_{b}}(\boldsymbol{x}_{b,n,t,i}, i, \boldsymbol{s}_{b,n,t})). 
\end{equation}

Finally, by employing reparameterization technique for the $p_{\boldsymbol{\theta}_{b}}(\boldsymbol{x}_{b,n,t,i-1}|\boldsymbol{s}_{b,n,t}, i, \boldsymbol{x}_{b,n,t,i})$, we can get the recurrent relationship as shown in Eqn. (\ref{Eq-update-rule}).
\end{proof}

\begin{remark}
The derivation formulation of action probability in Eqn. (\ref{Eq-update-rule}) follows the same diffusion processing as \cite{ho2020denoising}, with the only difference being that $x_{b,n,t,i}$ represents an action probability vector instead of an image pixel tensor. Both are real-valued vectors in Euclidean space, and the Gaussian noise addition and removal operations are defined identically. The softmax layer is applied to $x_{b,n,t,i}$ to obtain the final action probability distribution, ensuring that the output lies on the simplex of the action space.   
\end{remark}

\textbf{Latent Action Diffusion Strategy.} Existing diffusion-based DRL methods perform diffusion processing starting from Gaussian noise, which easily leads to divergence in decision-making, degrading their performance. To address this issue, therefore, we design a latent action probability strategy. Specifically, we introduce a latent action probability $\boldsymbol{x}_{b,n,t,I}$ instead of the random Gaussian noise for each $b \in \mathcal{B}$, $n \in \mathbb{N}_{b,t}$, and $t\in\mathcal{T}$ by considering tasks that usually have a specific periodic pattern over a certain period. For convenience, we assume $N=\max\{N_{b,t}\}_{b\in\mathcal{B},t\in\mathcal{T}}$ and create an array $X_b$ with length $N$ for each $b \in \mathcal{B}$ to store the historical action probability of task $n \in \mathbb{N}_{b,t}$, where each element $X_b[n]$ is initialized by a standard Gaussian distribution. Then, for each task $n\in\mathbb{N}_{b,t}$ at time slot $t$, the $\boldsymbol{x}_{b,n,t,I}$ is achieved by using $X_b[n]$. After executing the diffusion processing, we can get its output $\boldsymbol{x}_{b,n,t,0}$ and update $X_b[n]\leftarrow \boldsymbol{x}_{b,n,t,0}$. Furthermore, we also add the $\boldsymbol{x}_{b,n,t,I}$ to the transition tuple to improve the network training, i.e., using the $(\boldsymbol{s}_{b,n,t}$, $\boldsymbol{x}_{b,n,t,I}$, $\boldsymbol{a}_{b,n,t}$, $r_{b,n,t}$, $\boldsymbol{s}_{b,n,t}^{\text{next}}$, $\boldsymbol{x}_{b,n,t,I}^{\text{next}})$ instead of the $(\boldsymbol{s}_{b,n,t}$, $\boldsymbol{a}_{b,n,t}$, $r_{b,n,t}$, $\boldsymbol{s}_{b,n,t}^{\text{next}})$ in our method. Like this, $\boldsymbol{x}_{b,n,t,I}$ is further initialized by $\boldsymbol{x}_{b,n,t-1,0}$. Then, our method can quickly achieve the near-optimal decision by using the historical action probability, which ensures lower service delay than diffusion-based DRL methods. 

\begin{proposition}[Warm-start Advantage of Historical Action Probability Prior] 
\label{prop:warm_start}
Let $x_t^\star$ denote the ideal latent action probability vector corresponding to the optimal policy $\pi_t^\star$ under state $s_t$. Let $h_t=x_{t-1,0}$ denote the historical latent action probability prior adopted by LAD-TS, whereas conventional diffusion-based scheduling initializes it using $g_t\sim\mathcal{N}(0,I)$. Assume that: 1) the temporal drift satisfies $\|x_t^\star-x_{t-1}^\star\|_2\leq \delta_t$; 2) the previous latent estimation satisfies $\|x_{t-1,0}-x_{t-1}^\star\|_2\leq e_{t-1}$; 3) the $I$-step reverse diffusion mapping followed by softmax is locally $L_I$-Lipschitz up to an approximation error $\epsilon_{\theta}$. Then, when
$e_{t-1}+\delta_t<D_G(t)$, the historical action probability prior provides a smaller policy-error bound than Gaussian initialization.
\end{proposition}

\begin{proof}
By applying the triangle inequality, the distance between the historical latent action probability prior and the ideal latent action probability can be bounded as
\begin{align}
\|h_t-x_t^\star\|_2 &= \|x_{t-1,0}-x_t^\star\|_2 \nonumber\\
&\leq \|x_{t-1,0}-x_{t-1}^\star\|_2 + \|x_{t-1}^\star-x_t^\star\|_2 \nonumber\\
&\leq e_{t-1}+\delta_t .
\end{align}
Substituting the above inequality into the local Lipschitz property of the reverse diffusion mapping directly yields
\begin{equation}
\left\| \pi_t^{H}-\pi_t^\star \right\|_1
\leq L_I(e_{t-1}+\delta_t)+\epsilon_\theta .
\end{equation}

Similarly, applying the same Lipschitz condition to Gaussian initialization and taking expectation over $g_t$ gives
\begin{equation}
\mathbb{E} \left[ \left\| \pi_t^{G}-\pi_t^\star \right\|_1
\right]
\leq L_I D_G(t)+\epsilon_\theta .
\end{equation}
where $D_G(t)= \mathbb{E}_{g_t\sim\mathcal{N}(0,I)}
\left[ \|g_t-x_t^\star\|_2 \right]$.

Hence, the historical latent action probability prior achieves a smaller initialization-induced policy error whenever $e_{t-1}+\delta_t<D_G(t)$.
\end{proof}

\textbf{Network training.} To well ensure the effectiveness of diffusion processing, the $\boldsymbol{x}_{b,n,t}$ is also added to the transition tuple for network training in our method. Let $\tilde{\boldsymbol{\theta}}_{b}$, $\boldsymbol{\omega}_{b,1}$, $\boldsymbol{\omega}_{b,2}$, $\tilde{\boldsymbol{\omega}}_{b,1}$ and $\tilde{\boldsymbol{\omega}}_{b,2}$ denote the parameters of the s-LADN, CN 1, CN 2, TN 1, and TN 2 in ES $b$, respectively. Each ES has a trainer to perform network training using batch samples in the experience pool. The parameters of $\boldsymbol{\theta}_{b}$, $\tilde{\boldsymbol{\theta}}_{b}$, $\{\boldsymbol{\omega}_{b,j}\}_{j=1,2}$, and $\{\tilde{\boldsymbol{\omega}}_{b,j}\}_{j=1,2}$ are updated as follows. 

Firstly, $K$ samples are extracted from the experience pool: $\{[\boldsymbol{s}^{k}_{b,n,t}, \boldsymbol{x}^{k}_{b,n,t}, \boldsymbol{a}^{k}_{b,n,t}, r^{k}_{b,n,t}, \boldsymbol{s}^{k,\text{next}}_{b,n,t}, \boldsymbol{x}^{k,\text{next}}_{b,n,t}]\}_{k=1, 2, \cdots, K}$. Then, $r_{b,n,t}$, $\boldsymbol{s}_{b,n,t}^{\text{next}}$, and $\boldsymbol{x}_{b,n,t}^{\text{next}}$ are inputted to the t-LADN and TNs for $\boldsymbol{Q}_{\text{target}}$ calculation. More precisely, $Q^{k}_{\text{target}}$ = $r_{b,n,t}$ + $\gamma\cdot(\pi_{\boldsymbol{\theta}_{b}}(\boldsymbol{a}^{k,\text{next}}_{b,n,t}|\boldsymbol{s}^{k,\text{next}}_{b,n,t}, \boldsymbol{x}^{k,\text{next}}_{b,n,t}, I)  \cdot Q^{k,\text{min}}_{\text{target}} + \alpha_{b} \cdot H_{\text{next}}^{k}(\pi_{\boldsymbol{\theta}_{b}}))$ and $\boldsymbol{a}^{k,\text{next}}_{b,n,t}\sim\pi_{\boldsymbol{\theta}_{b}}(\cdot|\boldsymbol{s}^{k,\text{next}}_{b,n,t}, \boldsymbol{x}^{k,\text{next}}_{b,n,t}, I)$, where $\gamma$ is a reward decay factor and  $\alpha_{b}$ is a temperature coefficient, $ Q^{k,\text{min}}_{\text{target}} = \min_{j=1,2}\{Q_{\tilde{\boldsymbol{\omega}}_{b,j}} (\boldsymbol{s}^{k,\text{next}}_{b,n,t}, \boldsymbol{a}^{k,\text{next}}_{b,n,t})\}$. $H_{\text{next}}^k(\pi_{\boldsymbol{\theta}_{b}})$ is the entropy of the action probability distribution with the inputs of $\boldsymbol{s}_{b,n,t}^{\text{next}}$ and $\boldsymbol{x}_{b,n,t}^{\text{next}}$. We have $H_{\text{next}}^{k}(\pi_{\boldsymbol{\theta}_{b}})$ = $-\pi_{\boldsymbol{\theta}_{b}}(\boldsymbol{a}^{k,\text{next}}_{b,n,t}|\boldsymbol{s}^{k,\text{next}}_{b,n,t}, \boldsymbol{x}^{k,\text{next}}_{b,n,t}, I)^\top
\times \log\pi_{\boldsymbol{\theta}_{b}}(\boldsymbol{a}^{k,\text{next}}_{b,n,t}|\boldsymbol{s}^{k,\text{next}}_{b,n,t}, \boldsymbol{x}^{k,\text{next}}_{b,n,t}, I).$

Secondly, the $\{\boldsymbol{\omega}_{b,j}\}_{j=1,2}$ are updated by minimizing the following critic loss function: $\mathcal{L}_{Q}(\boldsymbol{\omega}_{b,j})=$
\begin{align}\label{Eq-omega-update}
    \frac{1}{K}\sum\nolimits_{k=1}^{K}\left(Q_{\boldsymbol{\omega}_{b,j}}(\boldsymbol{s}^{k}_{b,n,t}, \boldsymbol{a}^{k}_{b,n,t})- Q^{k}_{\text{target}}\right)^2.
\end{align}

Thirdly, the $\boldsymbol{\theta}_{b}$ is updated by minimizing the following actor loss function:
\begin{align}\label{Eq-theta-update}
    \mathcal{L}_{\pi}(\boldsymbol{\theta}_{b})= & \frac{1}{K}\sum\nolimits_{k=1}^{K} ( -\alpha_{b}\cdot H^{k}(\pi_{\boldsymbol{\theta}_{b}})\nonumber \\
    &-\pi_{\boldsymbol{\theta}_{b}}(\boldsymbol{a}^{k}_{b,n,t}|\boldsymbol{s}^{k}_{b,n,t}, \boldsymbol{x}^{k}_{b,n,t}, I) \cdot Q^{k}_{\text{eval}})^2,
\end{align}
where $Q^{k}_{\text{eval}}=\min_{j=1,2}\{Q_{\boldsymbol{\omega}_{b,j}}(s^{k}_{b,n,t}, \boldsymbol{a}^{k}_{b,n,t})\}$. $H^{k}(\pi_{\boldsymbol{\theta}_{b}})$ is the entropy of the action probability distribution with the inputs of $\boldsymbol{s}_{b,n,t}$ and $\boldsymbol{x}_{b,n,t}$. We have $
H^{k}(\pi_{\boldsymbol{\theta}_{b}})$ = $- \pi_{\boldsymbol{\theta}_{b}} (\boldsymbol{a}^{k}_{b,n,t}|\boldsymbol{s}^{k}_{b,n,t}, \boldsymbol{x}^{k}_{b,n,t}, I)^\top
\times \log\pi_{\boldsymbol{\theta}_{b}}(\boldsymbol{a}^{k}_{b,n,t}|\boldsymbol{s}^{k}_{b,n,t}, \boldsymbol{x}^{k}_{b,n,t}, I)$.

Fourthly, the $\alpha_{b}$ is updated by minimizing the entropy loss function for each BS $b \in \mathcal{B}$:
\begin{equation}\label{Eq-alpha-update}
    \mathcal{L}({\alpha_{b}}) = (-H^{k}(\pi_{\boldsymbol{\theta}_{b}}) - \tilde{H}) \cdot\alpha_{b},
\end{equation}
where $\tilde{H}$ is a hyper-parameter of target entropy.

Finally, the $\{\tilde{\boldsymbol{\omega}}_{b,j}\}_{j=1,2}$ are updated by the soft updating operation as:
\begin{equation}\label{Eq-tildeomega-update}
    \tilde{\boldsymbol{\omega}}_{b,j} =  \tau_{b} \cdot \boldsymbol{\omega}_{b,j} + (1-\tau_{b}) \cdot \tilde{\boldsymbol{\omega}}_{b,j},
\end{equation}
where the $\tau_{b}$ is the weight of the soft updating operation. The $\tilde{\boldsymbol{\theta}}_{b}$ is updated by copying $\boldsymbol{\theta}_{b}$.

\subsection{Algorithm Implementation}\label{sec4-4}
We implement our LAD-TS method as an online distributed algorithm, as shown in Algorithm \ref{LAD-TS-algorithm}. In this algorithm, the inputs are the AIGC tasks with their data $\{d_{n}, z_{n}\}_{n\in\mathbb{N}_{b,t}}$ and denoising step $I$. The output is the task allocation decision $\boldsymbol{\psi}$. The Algorithm \ref{LAD-TS-algorithm} works as follows.

\begin{algorithm}[!t]
\DontPrintSemicolon
\KwInput{The AIGC task data $\{d_{n}, z_{n}\}_{n\in\mathbb{N}_{b,t}, b\in\mathcal{B}, t\in\mathcal{T}}$; Denosing step $I$;}
\KwOutput{The task allocation decision $\boldsymbol{\psi}$;}
Initialize the array $X_{b}$ by a standard Gaussian distribution for each $b\in\mathcal{B}$;\\
Initialize the LADNs, CNs, and TNs with random $\tilde{\boldsymbol{\theta}}_{b}$, $\boldsymbol{\theta}_{b}$, $\{\boldsymbol{\omega}_{b,j}\}_{j=1,2}$, and $\{\tilde{\boldsymbol{\omega}}_{b,j}\}_{j=1,2}$ for each $b\in \mathcal{B}$;\\
Initialize the queue $q_{0,b'} = 0$ and the experience pool $\mathcal{R}_{b'}$ = $\emptyset$ for each $b' \in \mathcal{B}$;\\
\For{\rm{episode $\in\{1, 2,\cdots, E\}$}}
{
  Reset system environment;\\
  \ForEach{\rm{time slot $t\in \mathcal{T}$}}
    {   \For{\rm{all BS $b \in \mathcal{B}$ in parallel}}
        {
            \ForEach{\rm{new arrival task $n \in \mathbb{N}_{b,t}$}}
            {                  
                Observe $\boldsymbol{s}_{b,n,t}$ and $\boldsymbol{x}_{b,n,t,I}\leftarrow X_{b}[n]$;\\
                Generate $\boldsymbol{x}_{b,n,t,0}$ and $\boldsymbol{a}_{b,n,t}$ by the Actor model;\\
                Calculate $r_{b,n,t}$ using Eqn. (\ref{Eq-reward});\\
                Update $X_{b}[n] \leftarrow \boldsymbol{x}_{b,n,t,0}$;\\
                Observe $\boldsymbol{s}_{b,n,t}^{\text{next}}$ and $\boldsymbol{x}_{b,n,t,I}^{\text{next}}\leftarrow X^{\text{next}}_{b}[n]$;\\
                Store the tuple $(\boldsymbol{s}_{b,n,t}$, $\boldsymbol{x}_{b,n,t,I}$, $\boldsymbol{a}_{b,n,t}$, $r_{b,n,t}$, $\boldsymbol{s}_{b,n,t}^{\text{next}}$, $\boldsymbol{x}_{b,n,t,I}^{\text{next}})$ to $\mathcal{R}_{b}$;\\ 
            }       
            \If{\rm{$|\mathcal{R}_{b}|$ > 300}}
            {         
                Sample $K$ tuples from $\mathcal{R}_{b}$;\\
                Execute network training to update $\{\boldsymbol{\omega}_{b,j}\}_{j=1,2}$, $\boldsymbol{\theta}_{b}$, $\alpha_{b}$, and $\{\tilde{\boldsymbol{\omega}}_{b,j}\}_{j=1,2}$ using Eqns. (\ref{Eq-omega-update}) - (\ref{Eq-tildeomega-update});\\         
                Update $\tilde{\boldsymbol{\theta}}_{b}$ by copying $\boldsymbol{\theta}_{b}$;\\
            }     
        }
        Update $\boldsymbol{q}_{t}$ by the Eqn. (\ref{queue-update});\\
    }
}
\caption{Online distributed LAD-TS algorithm}
\label{LAD-TS-algorithm}
\end{algorithm}

\textbf{Step 1.} It begins by initializing each element $X_{b}[n]$ in the array $X_{b}$ with a standard Gaussian distribution for each task $n\in\mathbb{N}_{b,t}$ and $b\in\mathcal{B}$ (Line 1). Then, for each BS $b \in \mathcal{B}$, LADNs, CNs, and TNs are initialized with random parameters $\tilde{\boldsymbol{\theta}}_{b}$, $\boldsymbol{\theta}_{b}$, $\{\boldsymbol{\omega}_{b,j}\}_{j=1,2}$, and $\{\tilde{\boldsymbol{\omega}}_{b,j}\}_{j=1,2}$, respectively (Line 2). The queue $q_{0,b'}$ and the experience pool $\mathcal{R}_{b'}$ at each BS $b' \in \mathcal{B}$ are initialized to 0 and empty $\emptyset$, respectively (Line 3).

\textbf{Step 2.} For each episode = $1, 2, \cdots, E$, the system environment is reset. Then, for each time slot $t \in \mathcal{T}$ and each task $n \in \mathbb{N}_{b,t}$ in a parallel way (Lines 6 - 8), the system state $\boldsymbol{s}_{b,n,t}$ and latent action probability $\boldsymbol{x}_{b,n,t,I}$ are observed from the environment and $X_{b}$, respectively (Line 9). Then, by utilizing the $\boldsymbol{s}_{b,n,t}$ and $\boldsymbol{x}_{b,n,t,I}$, and $I$ as input, the action $\boldsymbol{a}_{b,n,t}$ and $ \boldsymbol{x}_{b,n,t,0}$ are generated by the Actor (Line 10). Accordingly, the reward $r_{b,n,t}$ is calculated using Eqn. (\ref{Eq-reward}) (Line 11). Furthermore, the $\boldsymbol{x}_{b,n,t,I}$ is updated by equaling to $\boldsymbol{x}_{b,n,t,0}$ (Line 12). The next system state $\boldsymbol{s}_{b,n,t}^{\text{next}}$ and latent action probability $\boldsymbol{x}_{b,n,t,I}^{\text{next}}$ are observed from the environment and $X_{b}$, respectively (Line 13). The transition tuple $(\boldsymbol{s}_{b,n,t}$, $\boldsymbol{x}_{b,n,t,I}$, $\boldsymbol{a}_{b,n,t}, r_{b,n,t}$, $\boldsymbol{s}_{b,n,t}^{\text{next}}$, $\boldsymbol{x}_{b,n,t,I}^{\text{next}})$ is stored to the experience pool $\mathcal{R}_{b}$ (Line 14). 

\textbf{Step 3.} If the size $|\mathcal{R}_{b}|$ of experience pool exceeds 300, $K$ samples from $\mathcal{R}_{b}$ are extracted to sequentially update the model parameters $\{\boldsymbol{\omega}_{b,j}\}_{j=1,2}$, $\boldsymbol{\theta}_{b}$, $\alpha_{b}$, and $\{\tilde{\boldsymbol{\omega}}_{b,j}\}_{j=1,2}$ (Lines 15 - 17). After the completion of network training, the parameter $\tilde{\boldsymbol{\theta}}_{b}$ is updated by copying $\boldsymbol{\theta}_{b}$ (Line 18). Here, $|\mathcal{R}_{b}| > 300$ is set to control that at least $300$ history tuples have been stored, ensuring model training based on user-defined real-world conditions \cite{haarnoja2018soft}.

\textbf{Step 4.} At the end of each time slot, the queue $\boldsymbol{q}_{t}$ is updated according to the Eqn. (\ref{queue-update}) (Line 19). 

Finally, the near-optimal task allocation policy $\pi_{\tilde{\boldsymbol{\theta}}}$ is gradually learned through the above steps. Here, for each time slot $t$, the time complexity of Algorithm \ref{LAD-TS-algorithm} is a linear time complexity $O(N_{b,t})$ as proved in the following Theorem \ref{Theorem3}.

\begin{theorem}\label{Theorem3}
For each time slot $t$, the time complexity of Algorithm \ref{LAD-TS-algorithm} is a linear time complexity $O(N_{b,t})$.
\end{theorem}

\begin{proof}
As shown in Algorithm \ref{LAD-TS-algorithm}, for each time slot $t \in \mathcal{T}$, there are two ``for'' loops orderly. In the first loop, since all BSs are executed in parallel, its time complexity is $O(1)$. In the second loop, the time complexity is determined by the number of tasks $N_{b,t}$, the solving time of ES selection, and the network training time. Furthermore, the time complexity of ES selection is $O(1)$ since the t-LADN model has been trained previously. Meanwhile, the network training is performed offline. Thus, the total time complexity of Algorithm \ref{LAD-TS-algorithm} is $O(N_{b,t})$ for each time slot $t$.
\end{proof}

\begin{table}[!t]
\caption{Default setting of environment parameter.}
\centering
\resizebox{\linewidth}{!}{
\begin{tabular}{cccc}
    \toprule
    Param. & Value & Param. & Value\\
    \midrule
    $B$ & 20 \cite{du2024diffusion} & $|\mathcal{T}|$ & 60 \cite{chu2023online}\\
    $N_{b,t}$ & [1, 50] \cite{chu2023online} &  $\Delta$ & 1.0 second \cite{tang2022deep} \\
    $\rho_{n}$ & [100, 300] cycles/bit \cite{fan2024collaborative} & $d_{n}$ & [2, 5] Mbits \cite{chu2023online} \\ 
    $v_{n,b',t}$ & [400, 500] Mbits/s \cite{fan2024collaborative} & $f_{b'}$ & [10, 50] GHz \cite{fan2024collaborative}  \\
    \bottomrule
\end{tabular}}
\label{table3}
\end{table}

\section{Performance Evaluation}\label{sec5}
In this section, we develop a simulator in Python to evaluate the effectiveness of our method with insightful analysis. 

\begin{table}[!t]
\caption{Default setting of model parameters.}
\centering
\begin{tabular}{cc}
    \toprule
    Parameters & Value\\
    \midrule
     Hidden layers of actor and CN & 2 full connection (20, 20)\\
    \midrule
     Learning rate $\eta_{a}$ & 1e-4 \\
     Learning rate $\eta_{c}$ & 1e-3 \\          
     Learning rate $\eta_{\alpha}$ & 3e-4 \\ 
     Reward decay's factor $\gamma$ & 0.95 \\
     Soft updating's weight $\tau$ & 0.005 \\
     Batch size $K$ & 64 \\
     Denoising step $I$ & 5\\
     Entropy temperature $\alpha$ & 0.05\\ 
     Target entropy $\tilde{H}$ & -1\\
     Number of episodes $E$ & 60\\
     Experience pool's size $|\mathcal{R}|$ & 1000\\
     Optimizer & Adam\\
    \bottomrule
\end{tabular}
\label{table4}
\end{table}

\subsection{Setup}
\textbf{Environment Parameters.} The default setting of our environment parameters is presented in Table \ref{table3}. The number of BSs ($B$) is set to $20$ \cite{du2024diffusion}. The number of tasks arriving at each BS during time slot $t$ ($N_{b,t}$) is randomly generated within the range $[1, 50]$ \cite{chu2023online}. The total number of time slots $|\mathcal{T}|$ is $60$, corresponding to a duration of $1$ minute, with each time slot $t$ lasting $1$ second. Task sizes $d_{n}$ are randomly chosen between $[2, 5]$ Mbits \cite{chu2023online}, while the size $\tilde{d}_{n}$ of the generated images is set between $[0.6, 1.0]$ Mbits, equivalent to a resolution of $512\times512$ pixels. The required image generation quality $z_{n}$ and computation resources $\rho_{n}$ are set within the ranges of $[1, 15]$ steps \cite{du2024diffusion} and $[100, 300]$ CPU cycles/step, respectively. The computational capacity of each ES, $f_{b'}$, ranges from $[10, 50]$ GHz \cite{fan2024collaborative}. The data transmission rates $v_{n,b',t}$ and $v_{b',n,t}$ are uniformly distributed between $[400, 500]$ Mbits/s \cite{fan2024collaborative}.

\textbf{Model Parameters.} The default setting of our model parameters is presented in Table \ref{table4}. We construct the DRL model using the PyTorch framework. The LADNs, CNs, and TNs are all configured as an MLP, each comprising one input layer, two hidden layers with $20$ neurons each, and one output layer. Consistent with the findings in \cite{du2024diffusion}, we set the reward decay factor $\gamma$ to $0.95$. The experience replay buffer $\mathcal{R}_{b}$ at each ES has a capacity of $1000$. The learning rates for the LADNs, CNs, and $\alpha_{b}$ are $1e-4$, $1e-3$, and $3e-4$, respectively. We use a batch size of $64$ and the Adam optimizer. The initial value of $\alpha_{b}$ is $0.05$, and the target entropy $\tilde{H}$ is $-1.0$. Lastly, the weight coefficient $\tau_{b}$ is set to $0.005$ and the denoising step $I$ is set to $5$ according to our experiments as shown in Fig. \ref{Fig8}. 

\subsection{Baselines}
We use two well-established DRL baselines, a diffusion-based DRL baseline, and a heuristic optimal baseline to compare performance with our method as follows.

\textbf{DQN-TS:} The DQN \cite{mnih2015human} is a well-known DRL method that has been widely applied to various fields. We faithfully implement the DQN-based Task Scheduling (DQN-TS) method as a baseline with the same parameter setup in our experiments.
    
\textbf{SAC-TS:} The SAC \cite{haarnoja2018soft} is the state-of-the-art DRL-based method. We further implement the SAC-based Task Scheduling (SAC-TS) method as another baseline with the same parameter setup in our experiments.

\textbf{D2SAC-TS:} The D2SAC-TS \cite{du2024diffusion} is the state-of-the-art task scheduling method based on the diffusion model and DRL. We keep the D2SAC-TS's hyperparameter setting from the \cite{du2024diffusion} to ensure that the D2SAC-TS method is faithfully implemented. Compared to our method's hyperparameter setting, the different parameters mainly include: the number of hidden layers and neurons, batch size, the number of episodes, and the size of the experience pool.

\textbf{Opt-TS:} An optimal method selects the most suitable ES to process each task, providing the theoretical upper bound on the performance of AIGC services, but it is infeasible since the scheduler has no way of knowing in advance the ESs' available compute and network resources in a real-word edge system. However, we can assume that the scheduler knows the available computing and network resources of all ESs before assigning tasks at each time slot in our simulator. Thus, the Opt-TS method can select the best ES for each task at each time slot by exhaustively evaluating all possible action combinations, representing an upper bound on task scheduling performance in our simulation experiments. Therefore, if the results of our method are closer to the Opt-TS’s results, our method is better. Notably, the Opt-TS method is only used in our simulation experiments.

\begin{figure}[!t]
    \centering
    \includegraphics[width=\linewidth]{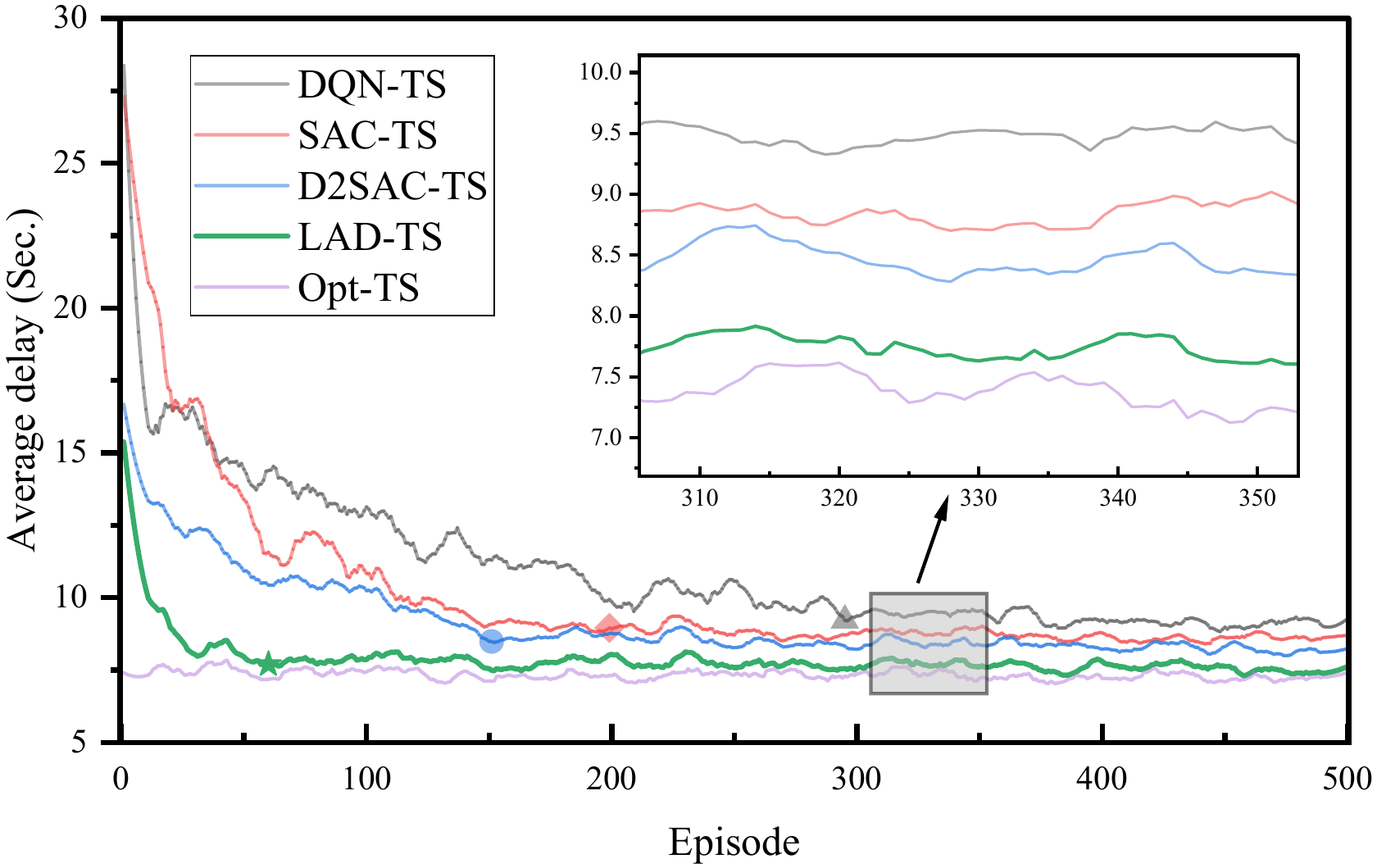}
    \caption{The average service delays of our LAD-TS method and baselines with increasing training episodes.}
    \label{Fig5}
\end{figure}

\subsection{Results}
This section presents extensive experimental results to assess the performance of our LAD-TS method compared to the baselines, as shown in Figs. \ref{Fig5} - \ref{Fig8} and Table \ref{table-ablation}. We conducted 50 independent experiments to achieve these results.

\subsubsection{Performance Comparison}
We first assess the learning performance of our method against four baselines, as depicted in Fig. \ref{Fig5}. The $x$-axis represents the episode count, while the $y$-axis indicates each episode's average service delay across tasks and time slots. Subsequently, we compare the average service delay of our method with the baselines by varying environment parameters as illustrated in Figs. \ref{Fig6} and \ref{Fig7}, and analyze the model parameters of the denoising step and entropy temperature impact on the service delay of our method as shown in Fig. \ref{Fig8}. 

\textbf{Learning Performance.} To gain insights into the learning performance of our LAD-TS method, we compare its performance with the baselines, as illustrated in Fig. \ref{Fig5}. With increasing training episodes, this figure shows that our LAD-TS method's delay initially decreases and then reaches convergence after only \textbf{60} training episodes. In contrast, the delays for the DQN-TS, SAC-TS, and D2SAC-TS methods require 300, 200, and 150 training episodes to converge, respectively. Consequently, the LAD-TS method reduces the number of training episodes by $\textbf{80\%}$ for DQN-TS, $\textbf{70\%}$ for SAC-TS, and $\textbf{60\%}$ for D2SAC-TS. This efficiency is attributed to the LAD-TS method's capability to make multi-step diffusion decisions with historical action probability guidance instead of Gaussian noise, accelerating the convergence process. On the other hand, from the subfigure in Fig. \ref{Fig5}, the average service delays for the DQN-TS, SAC-TS, D2SAC-TS, and LAD-TS methods stabilize at about $9.5$, $8.9$, $8.4$, and \textbf{7.7} seconds, respectively. The average service delay of the theoretical optimal method (Opt-TS) is $7.4$ seconds, which is only 0.3 seconds ($\textbf{4\%}$) lower than our method LAD-TS's result. These results show that our LAD-TS method significantly outperforms the DQN-TS, SAC-TS, and D2SAC-TS methods, and achieves a near-optimal service delay. Additionally, the line variation (blue color) of the average service delay of our method is smaller than that of the baselines (except Opt-TS), indicating the better performance stability of our method. These trends are observed because our LAD-TS method generates a better policy from accumulated policy priors by utilizing the diffusion model’s inherent advantages and the historical action information.

\begin{figure}[!t]
    \centering 
    \subfigure[Varying the number of tasks.]{
    \label{Fig6.sub.1}
    \includegraphics[width=0.8\linewidth]{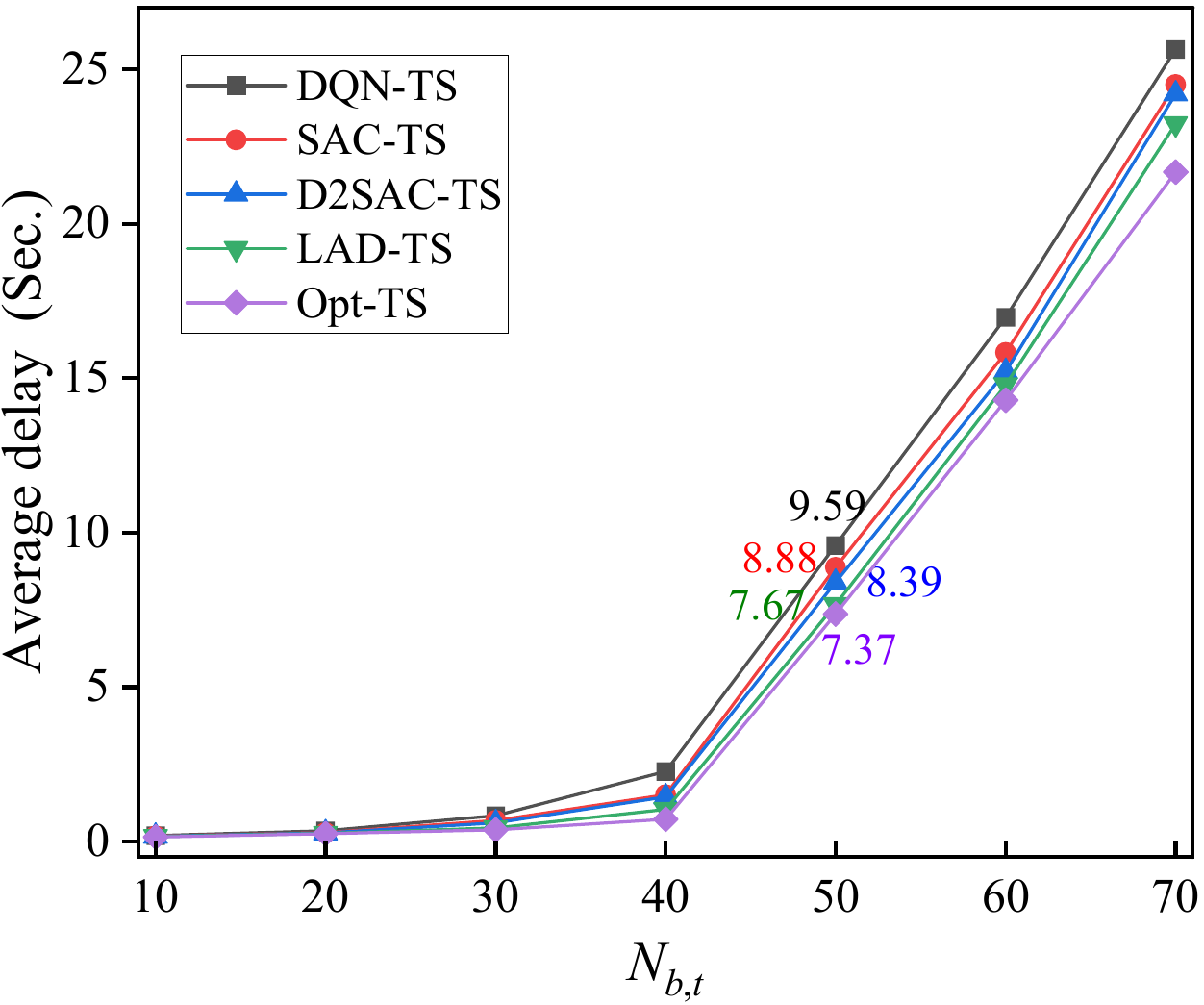}}
    
    \subfigure[Varying the ESs' capacities.]{
    \label{Fig6.sub.2}
    \includegraphics[width=0.8\linewidth]{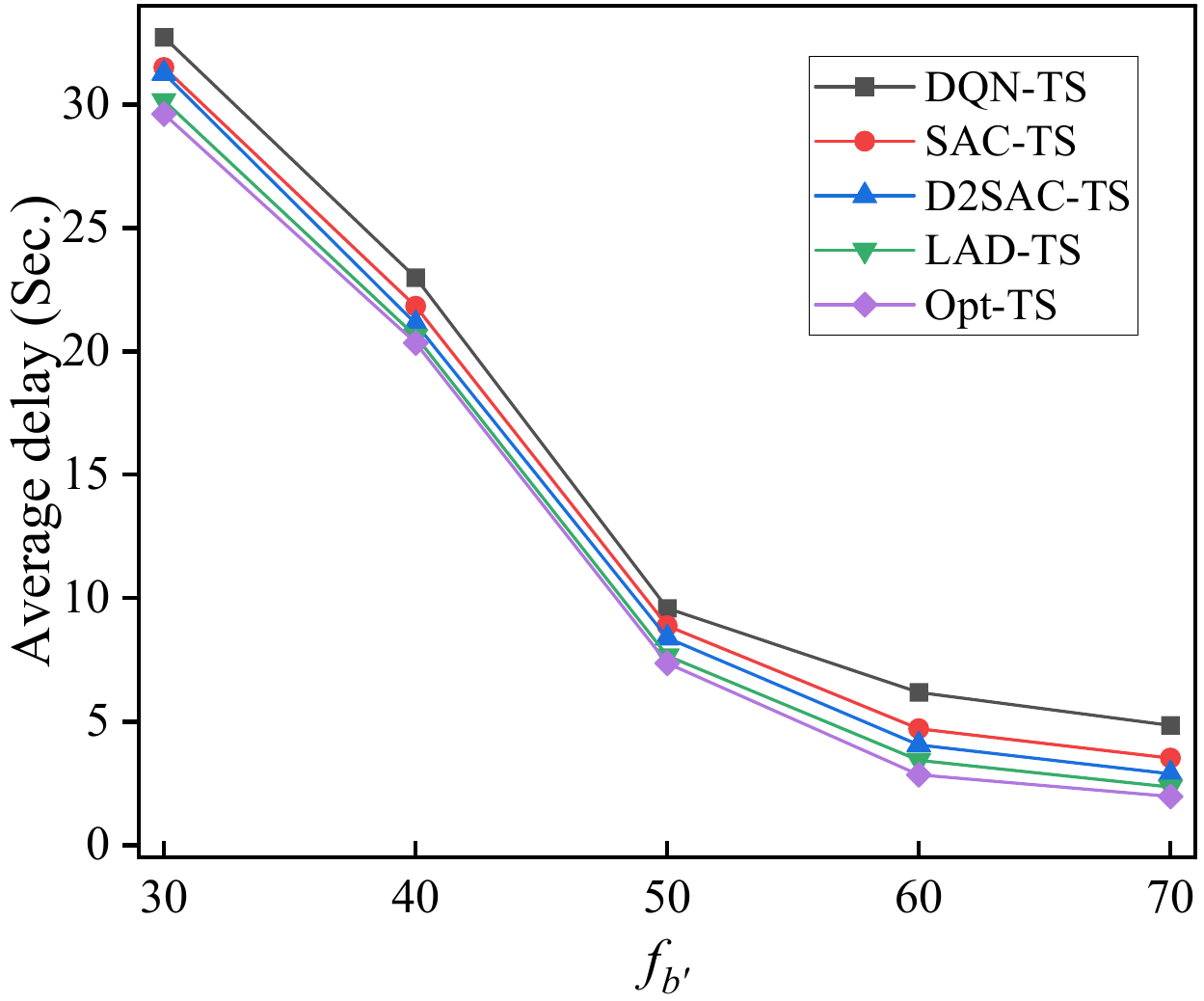}}
    \caption{The comparison results of our method and baselines by varying the number of tasks and the ESs' capacities.}
    \label{Fig6}
\end{figure}

\textbf{Effect of the Number of Tasks.} Fig. \ref{Fig6.sub.1} demonstrates how the service delays for all five methods escalate as the upper bound $N_{b,t}$ (representing the number of tasks arriving at each BS) increases from $10$ to $70$. This escalation is intuitive, i.e., a higher number of tasks leads to an increased workload being offloaded to the ESs, thereby elevating the service delay. Nonetheless, our LAD-TS method consistently exhibits lower service delays than the DQN-TS, SAC-TS, and D2SAC-TS methods, mirroring the optimal results as $N_{b,t}$ increases from $10$ to $70$. For instance, when $N_{b,t}$ equals 50, the average service delay of our LAD-TS method is only \textbf{7.67} seconds, which is significantly lower than those (\textbf{9.59}, \textbf{8.88}, and \textbf{8.39} seconds) of the DQN-TS, SAC-TS, and D2SAC-TS methods by \textbf{20.02\%}, \textbf{13.63\%}, and \textbf{8.58\%}, respectively. Also, this result is close to the Opt-TS method's service delay of \textbf{7.37} seconds. 

\begin{figure}[!t]
    \centering 
    \subfigure[Varying the AIGC qualities.]{
    \label{Fig7.sub.1}
    \includegraphics[width=0.8\linewidth]{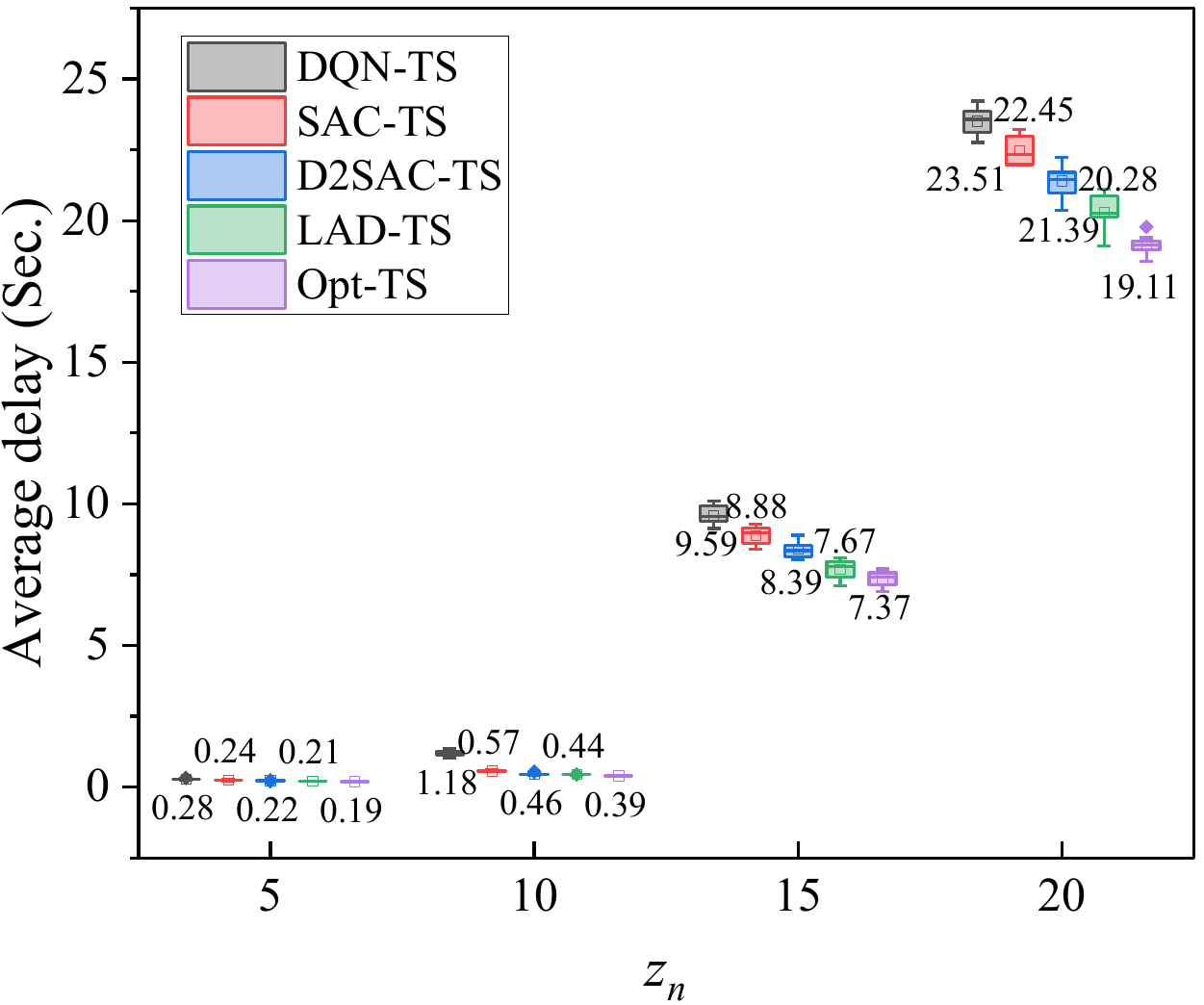}}
    
    \subfigure[Varying the number of BSs.]{
    \label{Fig7.sub.2}
    \includegraphics[width=0.8\linewidth]{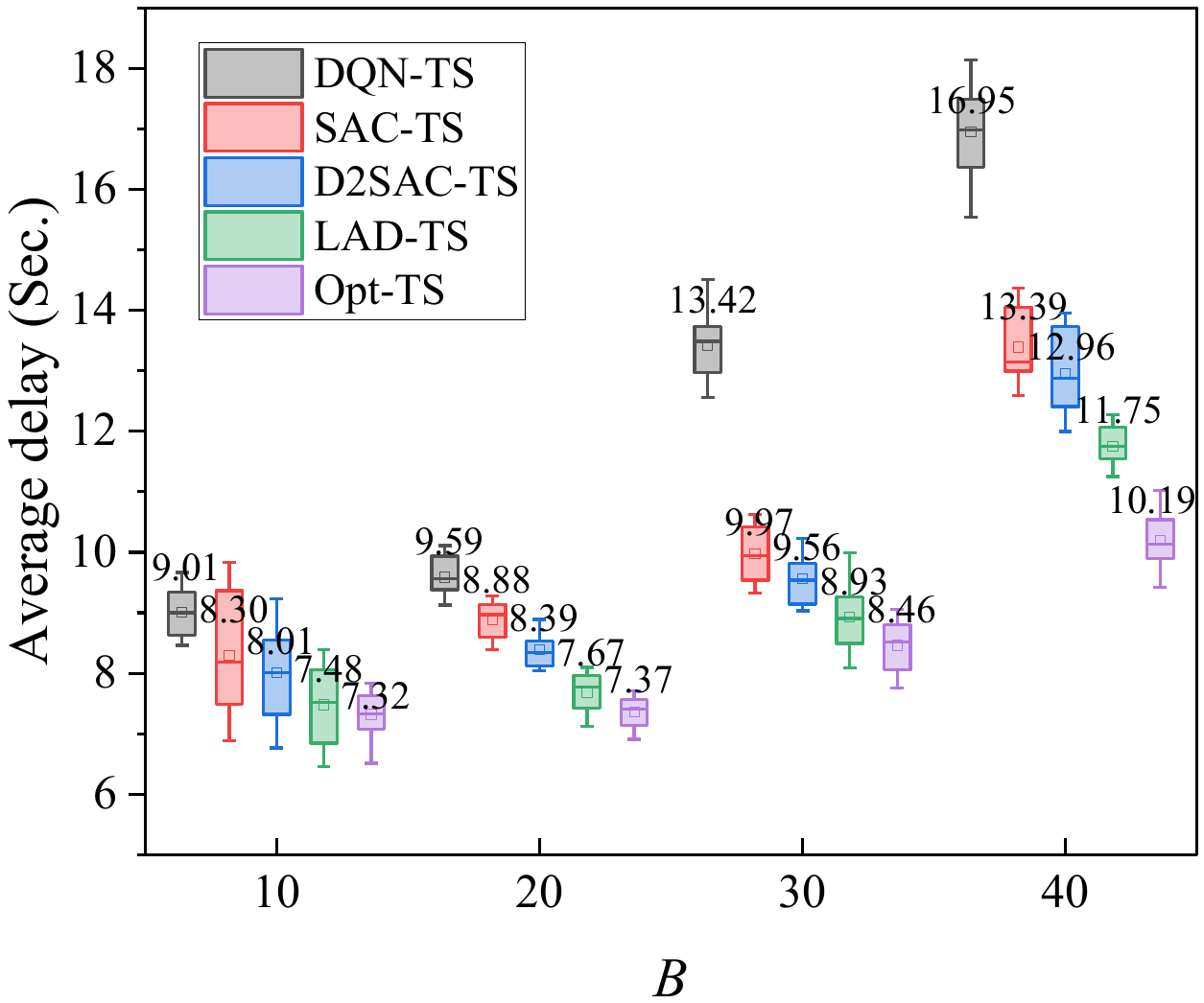}}
    \caption{The comparison results of our method and baselines by varying the AIGC qualities and the number of BSs.}
    \label{Fig7}
\end{figure}

\textbf{Effect of the Capacities of ESs.} Furthermore, we explore the impact of varying the upper bound $f_{b'}$ of ESs' capacities on service delay, as depicted in Fig. \ref{Fig6.sub.2}. From Fig. \ref{Fig6.sub.2}, we can see that the average service delays for all methods decrease with increasing $f_{b'}$ within the range $[30, 70]$. However, our LAD-TS method consistently outperforms the DQN-TS, SAC-TS, and D2SAC-TS methods, and near the Opt-TS method in service delay across the $f_{b'}$ variations.

\textbf{Effect of the Quality Demand of Image Generation.} The results presented in Fig. \ref{Fig7.sub.1} are achieved by varying the quality demand $z_{n}$ of image generation. Echoing the trends observed in Fig. \ref{Fig6.sub.1}, as $z_{n}$ increases from $5$ to $20$, our LAD-TS method consistently surpasses the DQN-TS, SAC-TS, and D2SAC-TS methods in minimizing service delay, closely aligning with the optimal Opt-TS method's results. Notably, at a high $z_{n}$ value of $20$, the average service delay for our LAD-TS method is only recorded at $20.28$ seconds, which is still close to the optimal Opt-TS method's result $19.11$ seconds. In contrast, the average service delays for the DQN-TS, SAC-TS, and D2SAC-TS methods are recorded at $23.51$, $22.45$, and $21.39$ seconds, respectively, which are significantly higher than that of our LAD-TS method.

\begin{figure}[!t]
    \centering 
    \subfigure[Our method's denoising step impact on service delay.]{
    \label{Fig8.sub.1}
    \includegraphics[width=0.9\linewidth]{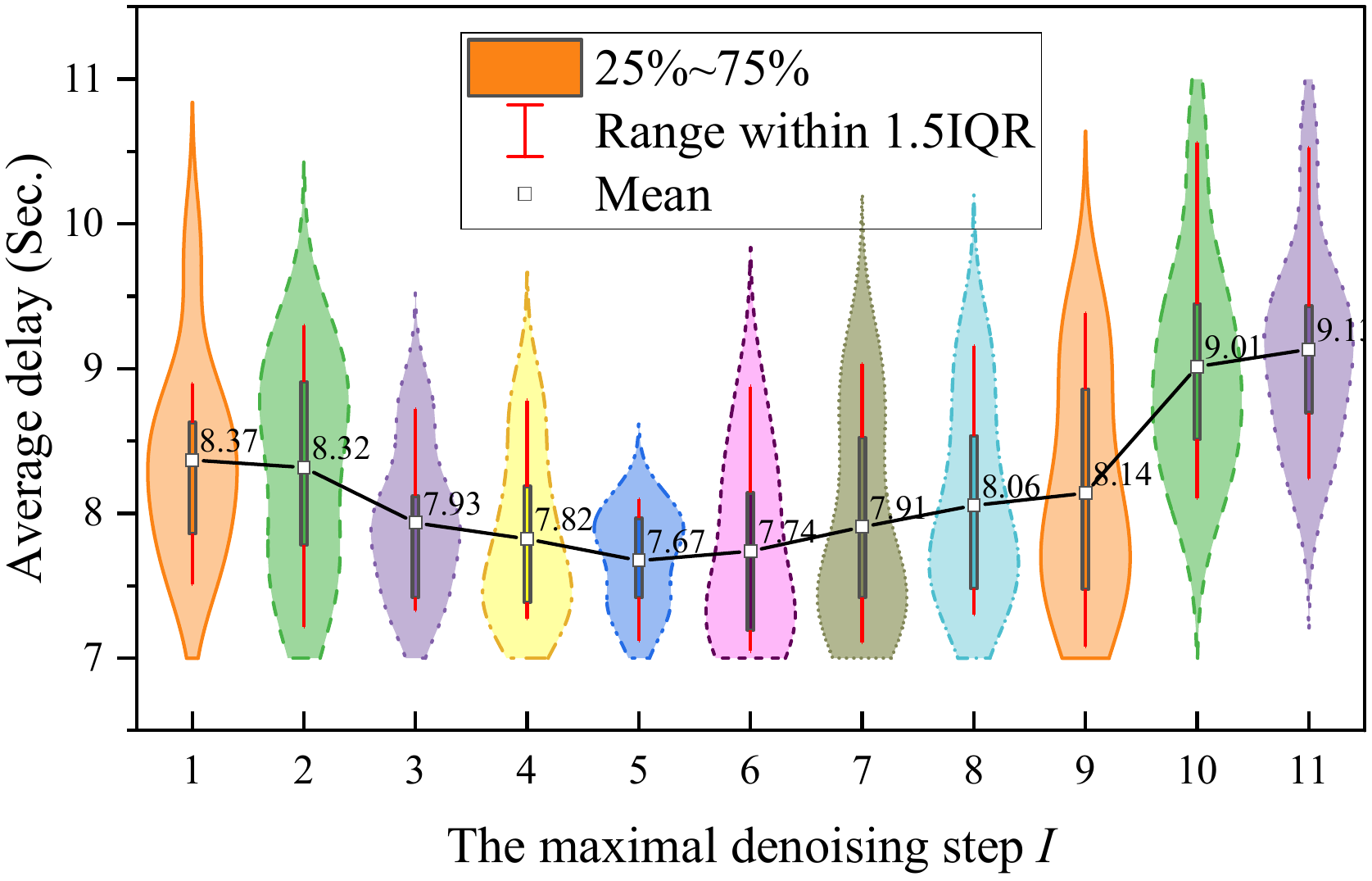}}
    
    \subfigure[Our method's entropy temperature impact on service delay.]{
    \label{Fig8.sub.2}
    \includegraphics[width=0.9\linewidth]{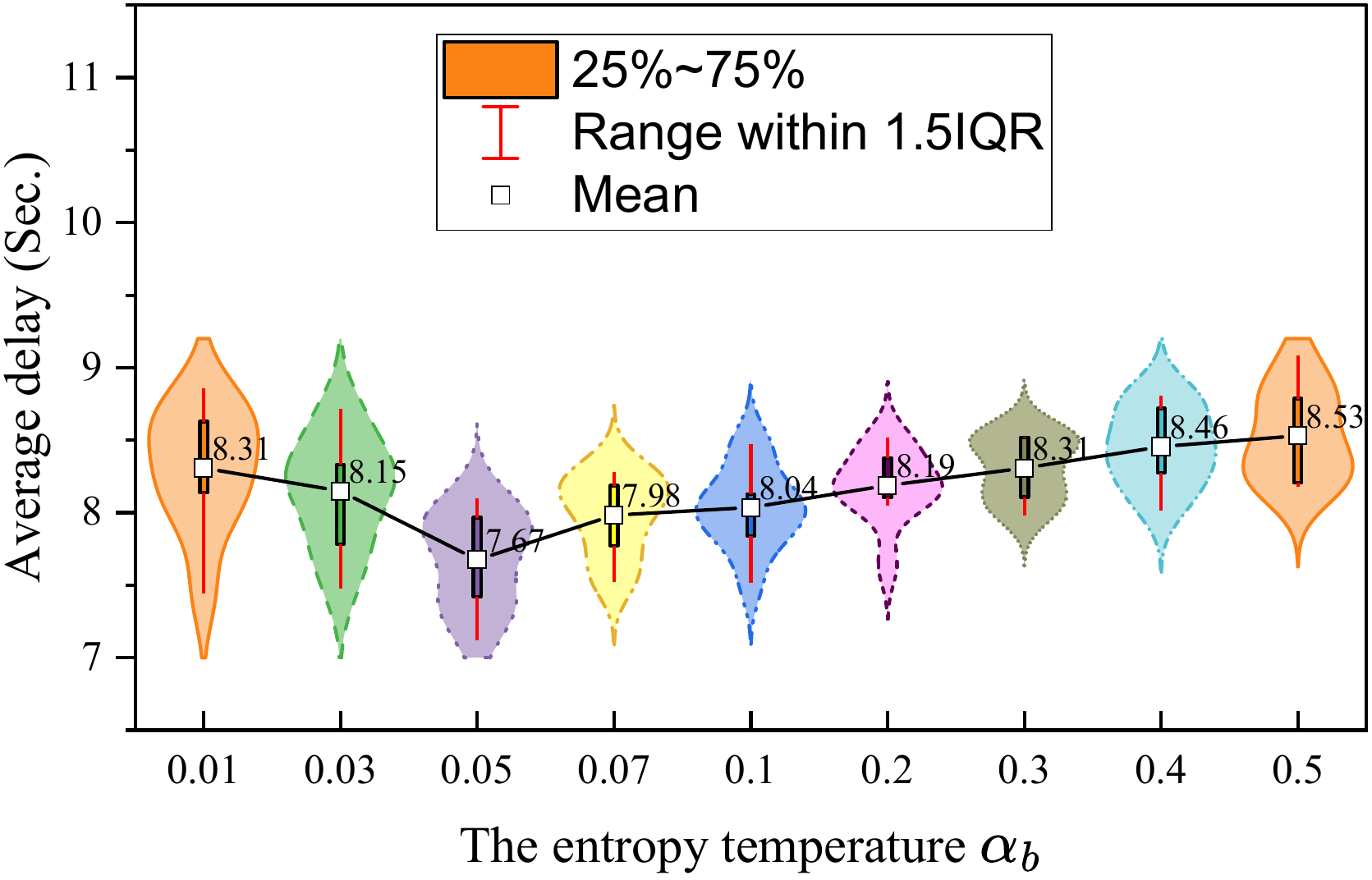}}
    \caption{The service delay of our method by varying denoising step $I$ and entropy temperature $\alpha_{b}$.}
    \label{Fig8}
\end{figure}

\textbf{Effect of the Number of BSs.} We also test the impact of varying the number $B$ of BSs on service delay, as illustrated in Fig. \ref{Fig7.sub.2}. Fig. \ref{Fig7.sub.2} reveals that service delays escalate for all methods as $B$ increases. Nonetheless, our LAD-TS method consistently outperforms the DQN-TS, SAC-TS, and D2SAC-TS methods, and is slightly higher than the optimal Opt-TS method's results across these variations. Moreover, as $B$ grows, the relative delay advantage of LAD-TS becomes more pronounced. When $B$ is set to $40$, the LAD-TS method achieves an average service delay of 11.75 seconds, which reduces the service delays (16.95, 13.39, and 12.96 seconds) of the DQN-TS, SAC-TS, and D2SAC-TS methods by $\textbf{30.67\%}$, $\textbf{12.25\%}$, and $\textbf{9.34\%}$, respectively. This superior performance is due to the LAD-TS method's ability to leverage the guidance from latent actions, which effectively optimizes the action policy to accommodate an increasing number of BSs with incoming tasks.

These findings prove that our LAD-TS method significantly outperforms existing state-of-the-art methods in reducing the service delay between $\textbf{8.58\%}$ to $\textbf{30.67\%}$ while closely approximating theoretical optimal policy and featuring the stability on different network conditions and AIGC workloads.

\subsubsection{Key Parameter Analysis}
Two key parameters of the denoising step $I$ and entropy temperature $\alpha_{b}$ are introduced to ensure computational efficiency and performance in our LAD-TS method. Specifically, setting the appropriate values for $I$ and $\alpha_{b}$ involves a trade-off between computational efficiency and performance. 

\textbf{Analysis of Parameter $I$.} Fig. \ref{Fig8.sub.1} presents the service delay of our method as the denoising step $I$ increases. Generally, a large value of $I$ is recommended to ensure image generation quality in the SD model, but this also leads to longer computation times \cite{ho2020denoising}. However, from Fig. \ref{Fig8.sub.1}, we can see that when $I$ equals 5, our LAD-TS method achieves the lowest service delay and maintains a small error distribution.

\textbf{Analysis of Parameter $\alpha_{b}$.} In DRL, the action entropy regularization term is introduced to balance exploration and exploitation \cite{haarnoja2018soft}, with its behavior controlled by the entropy temperature value $\alpha_{b}$. To assess its impact on the performance of our LAD-TS method, we evaluate the average service delay by varying the entropy temperature $\alpha_{b}$, as shown in Fig. \ref{Fig8.sub.2}. The results reveal that as $\alpha_{b}$ increases from 0.01 to 0.5, the average service delay first decreases and then begins to rise. The optimal value for minimizing service delay is found to be $\alpha = 0.05$, which strikes a balance between exploration and exploitation. Specifically, lower values of $\alpha_{b}$ lead to overly greedy behavior, restricting the agent from exploring actions with higher uncertainty, which can prevent the discovery of more effective actions. On the other hand, higher values of $\alpha$ encourage excessive randomness, potentially hindering the agent’s ability to converge on an optimal policy efficiently.

The above findings suggest that the moderate values of $\alpha = 0.05$ and $I = 5$ facilitate faster convergence and better performance by promoting effective exploration-exploitation and computational-performance trade-offs. This balance is crucial for ensuring the efficiency and optimality of the task scheduling policy learned by our LAD-TS method.

The above findings suggest that the moderate values of $\alpha = 0.05$ and $I = 5$ facilitate faster convergence and better performance by promoting effective exploration-exploitation and computational-performance trade-offs. This balance is crucial for ensuring the efficiency and optimality of the task scheduling policy learned by our LAD-TS method.

\subsubsection{Ablation Experiments}
To fully validate the proposed LAD-TS method's effectiveness, we perform an ablation experiment to analyze the performance effects of the SAC-TS and LAD-TS methods with or without the proposed latent action diffusion strategy, as shown in Table \ref{table-ablation}. We name the SAC-TS-U and LAD-TS-NU to represent the SAC-TS method using the proposed latent action diffusion strategy (i.e., the SAC-TS method's system state $\boldsymbol{S}$ incorporates historical latent action probabilities) and the LAD-TS method not using the proposed latent action diffusion strategy, respectively. From Table \ref{table-ablation}, the average service delay and convergence episode of our LAD-TS method are 7.67$\pm$0.64 and 60, respectively, which are lower than those of the LAD-TS-NU method by 5.66\% and 33.33\%, respectively. Meanwhile, the SAC-TS-U method also outperforms the SAC-TS in both service delay and convergence episode. The results of the LAD-TS-NU and LAD-TS methods are significantly better than the results of the SAC-TS and SAC-TS-U methods, respectively. These improvements can be attributed to the proposed latent action diffusion strategy and LADN model. The latent action diffusion strategy captures historical task information. The LADN model then incorporates this strategy to enable better generalization across various task scenarios, leading to improved service delay and convergence episode.

\begin{table}[!t] 
\centering  
\caption{Performance and comparison results of the SAC-TS, SAC-TS-U, LAD-TS-NU, and LAD-TS methods.}
\label{table-ablation}
\resizebox{\linewidth}{!}{
\begin{tabular}{l|cc}
    \toprule
    \textbf{Method} & \textbf{Average service delay (Sec.)} & \textbf{Converged episode}\\
    \midrule		
    SAC-TS & 8.88$\pm$1.43  & 200 \\ 
    SAC-TS-U & 8.54$\pm$1.21  & 150\\
    \midrule
    LAD-TS-NU & 8.13$\pm$0.82  & 90 \\
    LAD-TS & 7.67$\pm$0.64  & 60 \\	
    \bottomrule
\end{tabular}}
\end{table}

\section{Implementation and Validation}\label{sec6}
In this section, we develop a prototype to implement our DEdgeAI as shown in Fig. \ref{Fig9}. The practicality of our LAD-TS method is further validated in this system. The test-results are presented in Fig. \ref{Fig11}, Fig. \ref{Fig12}, and Tables \ref{table5}-\ref{table7}.

\begin{figure}[!t]
    \centering
    \includegraphics[width=0.7\linewidth]{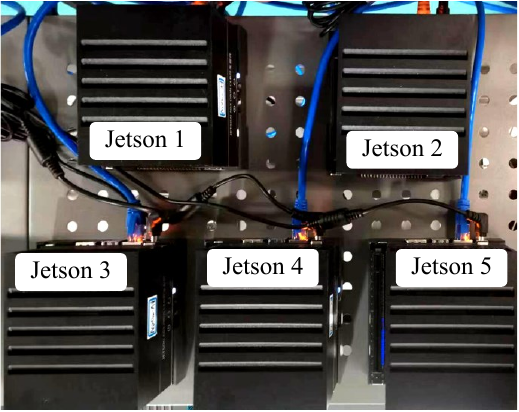}
    \caption{DEdgeAI is implemented on five Jetsons.}
    \label{Fig9}
\end{figure}

\subsection{System Implementation and Algorithm Deployment}
As depicted in Fig. \ref{Fig9}, our DEdgeAI system is implemented on five Jetson AGX Orin devices. Each device is equipped with the NVIDIA Ampere architecture GPU, featuring 2048 CUDA cores and 64 Tensor Cores. These devices serve as the ESs in our DEdgeAI system. The scheduler integrated with our Algorithm \ref{LAD-TS-algorithm} is also deployed on each Jetson. The devices are interconnected through a wired Gigabit local area network, enabling the DEdgeAI system to process incoming AIGC requests collaboratively in a distributed manner. 
 
We select the diffusion-based AIGC model as a representative example to validate our LAD-TS method in the DEdgeAI system. Specifically, we first observe that the original SD3 medium model has 8 billion parameters \cite{esser2024scaling}. Moreover, it uses three pretrained text encoders (i.e., OpenCLIP-ViT/G, CLIP-ViT/L, and T5xxl) to encode text representations, and an improved autoencoding model to encode image tokens, which require large memory and leads to a low inference efficiency for edge devices. To address these issues, we then implement the reSD3-m model by removing the T5xxl encoder of the original SD3 medium model \cite{sd3medium}. Finally, we deploy the reSD3-m model into each Jetson for providing text-to-image or image-to-image services.

\begin{figure}[!t]
    \centering
    \includegraphics[width=0.8\linewidth]{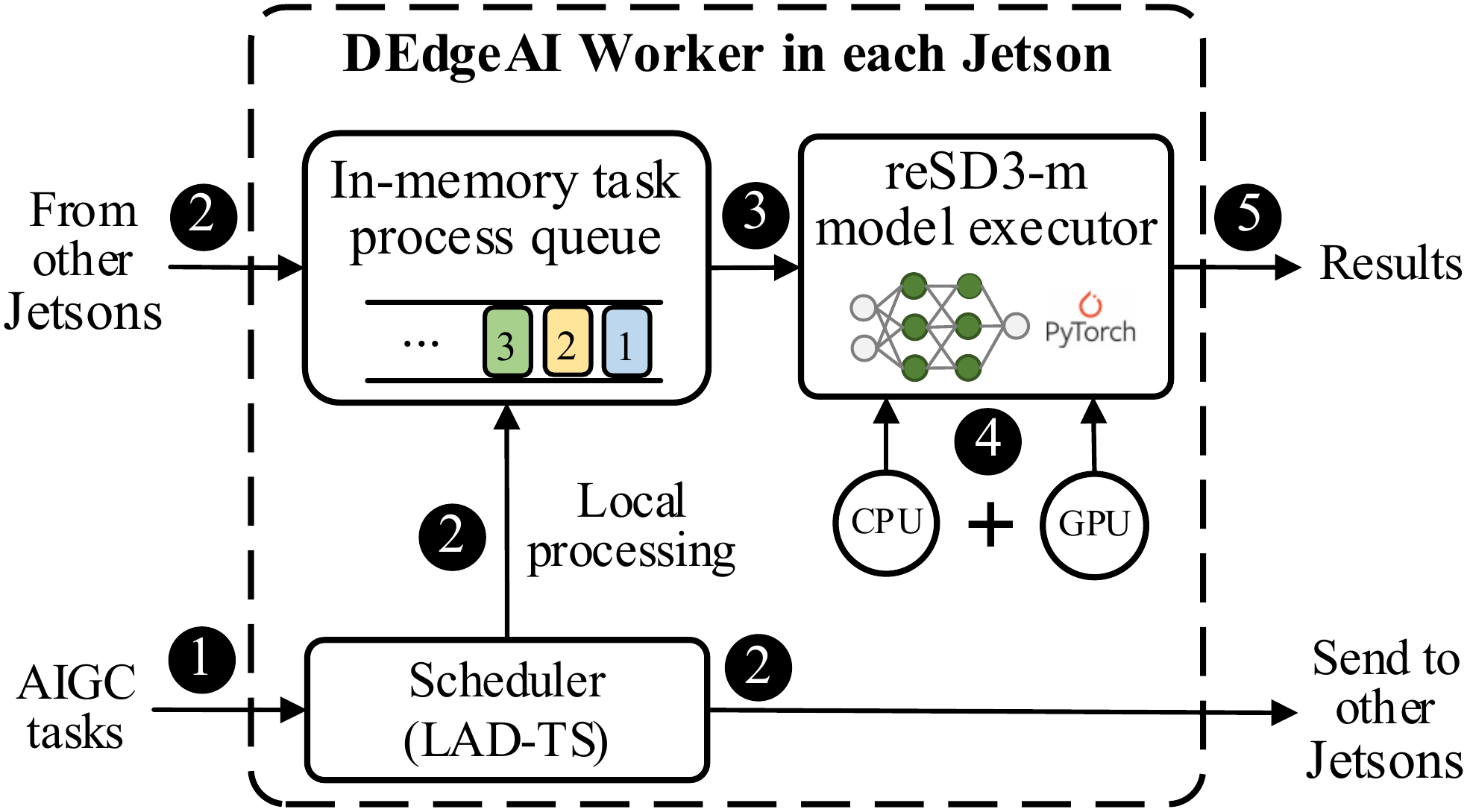}
    \caption{An instance of DEdgeAI worker in each Jetson.}
    \label{Fig10}
\end{figure}

\subsection{Experimental Design}
\textbf{Dataset}. We utilize the publicly available datasets of Flickr8k and COCO-2014 \cite{lin2014microsoft}. In particular, the text labels of the Flickr8k dataset are extracted as the input of the text-to-image task. The text labels of the Flickr8k dataset and the images of the COCO-2014 dataset are extracted as the input of the image-to-image task.

\textbf{DEdgeAI Worker}. As shown in Fig. \ref{Fig10}, for each Jetson device $b$ at time slot $t$, we randomly select $N_{b,t}$ text labels from the Flickr8k dataset and $N_{b,t}$ images from the COCO-2014 dataset to prompt AIGC tasks (\ding{182}). The scheduler on each Jetson device then employs our LAD-TS algorithm to offload these tasks so they can be processed successfully (\ding{183}, \ding{184}, and \ding{185}). Finally, the results are achieved (\ding{186}). The variables $d_{n}$ and $\tilde{d}_{n}$ represent the actual lengths of the AIGC task's input (a text, or a text and an image) and the processed output (an image), respectively. We calculate the service delay as the elapsed time from initiating task transmission to receiving the processed result.

\subsection{Test-bed Results}
\textbf{Service Delay.} As shown in Fig. \ref{Fig11}, all the methods' service delays increase with the increase of the number of text-to-image tasks $|\mathcal{N}|$. However, our LAD-TS method significantly outperforms the three baselines regarding total service delay under all the different $|\mathcal{N}|$. Specifically, Fig. \ref{Fig11} shows that when the $|\mathcal{N}|$ increases from 10 to 1000, the total service delay of our AMCoEdge method is increased from 62.23 to 3895.42 seconds. In contrast, the total service delays of DQN-TS, SAC-TS, and D2SAC-TS methods are increased from 79.49 to 4352.35 seconds, 73.37 to 4328.14 seconds, and 68.16 to 4103.37 seconds, respectively, which are higher than that of our AMCoEdge method by the average values \textbf{18.11\%}, \textbf{11.87\%}, and \textbf{7.36\%}, respectively.

\begin{figure}[!t]
    \centering
    \includegraphics[width=0.9\linewidth]{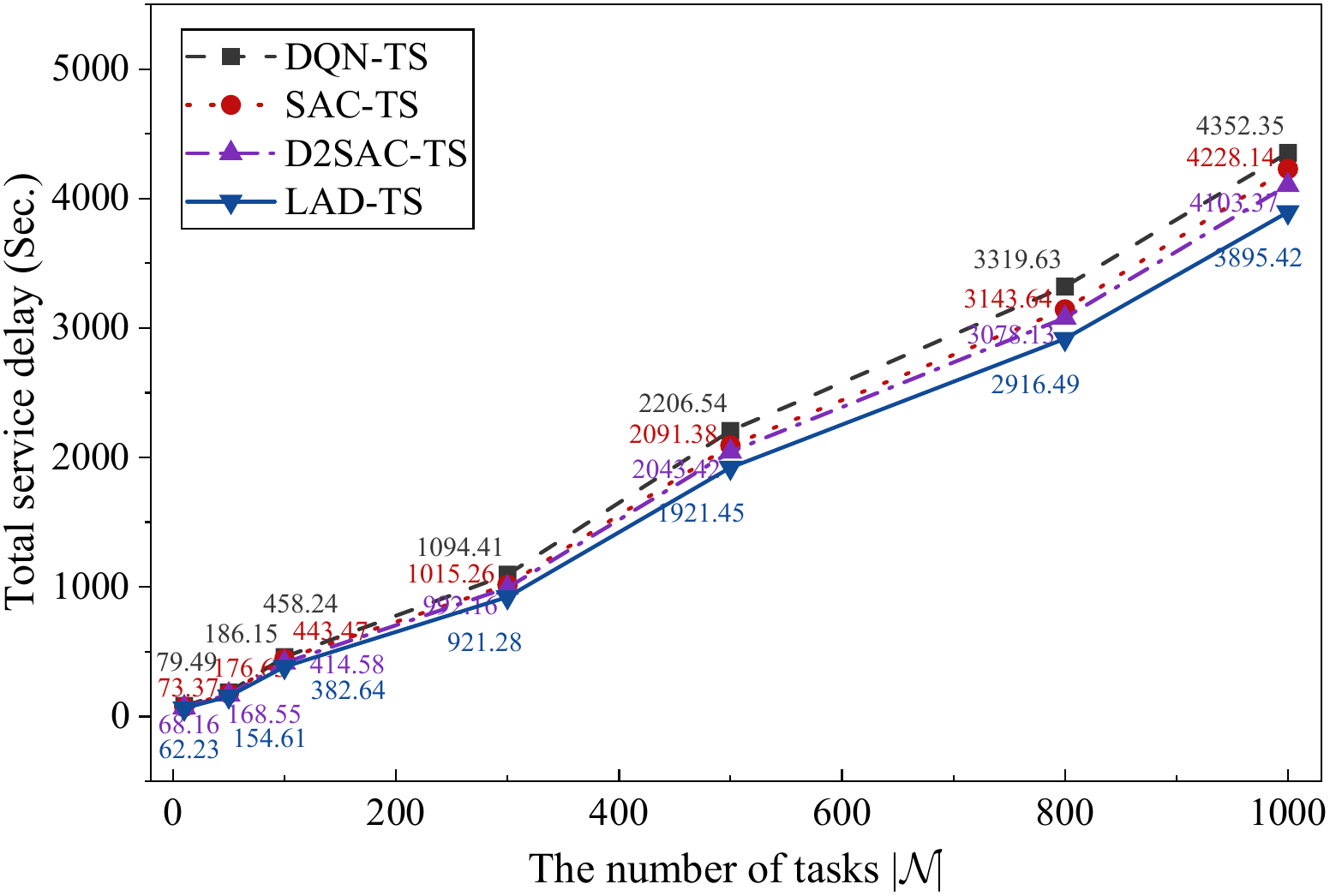}
    \caption{The results of the total service delay of our method and baselines in our developed DEdgeAI system by varying the number of text-to-image tasks $\mathcal{N}$.}
    \label{Fig11}
\end{figure}

\textbf{Generalization Ability.} We further evaluate the generalization ability of our proposed method under different AIGC model types, AIGC task types, and edge Hardware Configurations (HCs). In particular, we randomly select 10 text labels from the Flickr8k dataset as the input for text-to-image tasks, and randomly select 10 text labels from the Flickr8k dataset and 10 images from the COCO-2014 dataset as the input for image-to-image tasks. We set the CPU capacity of two out of the five Jetson devices to half of their original capacity, and used this as a heterogeneous HC in our experiments. Table \ref{table7} compares our proposed LAD-TS method and three baselines in terms of the total service delay for creating 10 images under two model types, two task types, and two HCs. However, our method always achieves the lowest service delay under individual domain shifts, demonstrating that LAD-TS learns a resource-aware scheduling policy rather than a policy specific to a single image-generation model or Jetson configuration.

\begin{table}[!t]
    \caption{The service delay results of the LAD-TS method and baselines in our developed DEdgeAI system under different AIGC models, AIGC tasks, and edge HCs.}
    \centering
    \resizebox{\linewidth}{!}{
    \begin{threeparttable}       
    \begin{tabular}{c|ccccc} 
        \toprule        
        \textbf{Model} &\textbf{\makecell{Task \\ category}} & \textbf{\makecell{Hardware \\ configuration}} & \textbf{Method} & \textbf{\makecell{Delay \\ (Sec.)}} & \textbf{\makecell{Improved \\ delay}}\\
        \midrule
        \multirow{12}{*}{SD3 med.}  & \multirow{4}{*}{Text-to-image} & \multirow{4}{*}{Original HC\tnote{1}} & DQN-TS & 161.45 & 20.45\% \\
        & & & SAC-TS & 151.12 & 15.01\% \\ 
        & & & D2SAC-TS & 139.96 & 8.24\% \\
        & & & LAD-TS & 128.43 & NA\\
        \cmidrule{2-6}
        & \multirow{4}{*}{Image-to-image} & \multirow{4}{*}{Original HC\tnote{1}} & DQN-TS & 190.67 & 19.86\% \\
        & & & SAC-TS & 178.24 & 14.27\% \\ 
        & & & D2SAC-TS & 166.35 & 8.14\% \\
        & & & LAD-TS & 152.81 & NA\\
        \cmidrule{2-6}
        & \multirow{4}{*}{Image-to-image} & \multirow{4}{*}{Heterg. HC\tnote{2}} & DQN-TS & 230.44 & 18.74\% \\
        & & & SAC-TS & 212.42 & 11.84\% \\ 
        & & & D2SAC-TS & 202.91 & 7.71\% \\
        & & & LAD-TS & 187.26 & NA\\
        \midrule
        \multirow{12}{*}{reSD3-m}  & \multirow{4}{*}{Text-to-image} & \multirow{4}{*}{Original HC\tnote{1}} & DQN-TS & 79.49 & 21.71\% \\
        & & & SAC-TS & 73.37 & 15.18\% \\ 
        & & & D2SAC-TS & 68.16 & 8.70\% \\
        & & & LAD-TS & 62.23 & NA\\
        \cmidrule{2-6}
        & \multirow{4}{*}{Image-to-image} & \multirow{4}{*}{Original HC\tnote{1}} & DQN-TS & 92.11 & 20.52\%\\
        & & & SAC-TS & 85.95 & 14.82\% \\ 
        & & & D2SAC-TS & 79.78 & 8.23\% \\
        & & & LAD-TS & 73.21 & NA \\
        \cmidrule{2-6}
        & \multirow{4}{*}{Image-to-image} & \multirow{4}{*}{Heterg. HC\tnote{2}} & DQN-TS & 111.81 & 18.09\% \\
        & & & SAC-TS & 104.73 & 12.55\% \\ 
        & & & D2SAC-TS & 99.47 & 7.93\% \\
        & & & LAD-TS & 91.58 & NA \\
        \bottomrule
    \end{tabular}
    \label{table7}    
    \begin{tablenotes}
        \footnotesize 
        \item[1] Five Jetson devices keep their original CPU capacities in the experiment.
        \item[2] Two out of the five Jetson devices are set to half of their original CPU capacities in the experiment.
    \end{tablenotes}
    \end{threeparttable}}   
\end{table}

\textbf{Running Time.} We record the running time of the designed reSD3-m model and the original SD3 medium model for generating an image in our developed DEdgeAI system, as shown in Table \ref{table5}. Experimental results show that the SD3 medium model generates an image in \textbf{35.68} seconds. In contrast, the reSD3-m model takes only \textbf{18.32} seconds to generate an image, significantly outperforming the SD3 medium model's running time (\textbf{35.68} seconds) by \textbf{48.65\%}.

\textbf{Memory Occupation.} We test the memory occupation of the reSD3-m model and the original SD3 medium model in our DEdgeAI system, as shown in Table \ref{table5}. Experimental results show that using the SD3 medium model occupies about \textbf{40 GB} of memory. However, when the reSD3-m model is used, the DEdgeAI memory is only occupied by about \textbf{16 GB}, reducing the memory usage by \textbf{60\%} compared to the original SD3 medium model.

\begin{table}[!t]
    \caption{The results of the running time, memory occupation, and image quality of the implemented reSD3-m and SD3 medium models in our developed DEdgeAI system.}
    \centering
    \resizebox{\linewidth}{!}{
    \begin{tabular}{c|cccc} 
        \toprule
        \textbf{\multirow{2}{*}{Model}} & \textbf{\multirow{2}{*}{Running time (Sec.)}} & \textbf{\multirow{2}{*}{Memory (GB)}} & \multicolumn{2}{c}{\textbf{Image quality}} \\
        \cline{4-5}
        & & & \textbf{FID} $\downarrow$ & \textbf{CLIP} $\uparrow$ \\
        \midrule
        SD3 med. & 35.68 & 40 & 50.12 & 0.231 \\   
        reSD3-m & 18.32 & 16 & 56.38 & 0.226\\
        \bottomrule
    \end{tabular}}
    \label{table5}
\end{table}

\textbf{Image Quality.} We also examine the image generation quality of the reSD3-m and SD3 medium models in our DEdgeAI system. Here, we use the FID and CLIP metrics to evaluate the generated image quality of the reSD3-m and SD3 medium models on the COCO-2014 validation split, as shown in Table \ref{table5}. The selective results of image generation are depicted in Fig. \ref{Fig12}. From Table \ref{table5} and  Fig. \ref{Fig12}, the reSD3-m model has a slightly higher FID score (indicating lower image quality), a slightly lower CLIP score (indicating worse alignment consistency), and poorer human visuals compared to the SD3 medium model. However, these image results satisfactorily meet the task requirements, further confirming the practical applicability of our LAD-TS method in real-world edge systems. Notably, this study focuses on the optimization of service delay, rather than the optimization of image generation quality. The different schedulers in our system don’t affect the quality of image generation.

\begin{figure}[!t]
    \centering
    \includegraphics[width=0.98\linewidth]{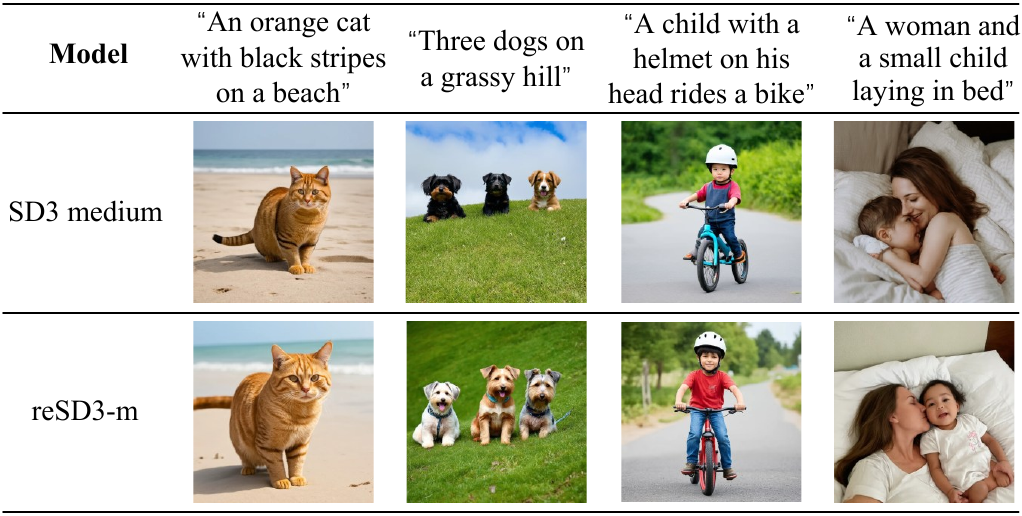}
    \caption{The partial experimental results for AIGC tasks by the implemented reSD3-m and original SD3 medium models in our DEdgeAI system. Intuitively, these results align well with the task's requirements. Importantly, the running time and memory usage of the reSD3-m model is significantly lower than that of original SD3 medium model, as shown in Table \ref{table5}.}
    \label{Fig12}
\end{figure}

\textbf{Platform Comparison.} We compare the total generation delay of our DEdgeAI system with that of existing representative platforms across varying task counts, as detailed in Table \ref{table6} (The data of Midjourney \cite{midjourney}, OpenAI \cite{dalle}, Replicate \cite{replicate}, Together.AI \cite{togetherai}, Stability.AI \cite{stablityai}, and Fal.AI \cite{falai} platforms are achieved from https://artificialanalysis.ai/text-to-image). The results indicate that for a single task, DEdgeAI exhibits a higher image generation delay than OpenAI, Replicate (using SD3 medium model), Together.AI, Stability.AI, and Fal.AI, due to the comparatively lower AI performance of the edge Jetson device used in our DEdegAI. However, our DEdgeAI outperforms Midjourney and Replicate (using SD1.5 model) even with a single task. Notably, as the number of AIGC tasks increases to greater than 100, DEdgeAI always achieves the lowest image generation delay relative to all five platforms. Moreover, when the number of AIGC tasks equals 100, DEdgeAI reduces the service delays by $\textbf{34.07\%}$ to $\textbf{94.96\%}$, compared to these representative platforms. This is because both LAD-TS and parallel processing are designed in our DEdgeAI system, significantly reducing service delays for large task requests. Notably, this comparison is not intended to suggest a one-to-one replacement. Instead, it highlights the potential of our DEdgeAI system in providing low-latency and adaptive scheduling under resource-limited decentralized edge environments.

\begin{table}[!t]
    \caption{The experimental results of the total generation delay of our DEdgeAI system and existing representative platforms under different numbers of AIGC tasks.}
    \centering
    \resizebox{\linewidth}{!}{
    \begin{tabular}{l|l|cccc} 
        \toprule
        \textbf{Platform} & \multirow{2}{*}{\textbf{Model}} & \multicolumn{4}{c}{\textbf{Median generated delay (Sec.)}}\\
        \cline{3-6}
        \textbf{/System} & & $|\mathcal{N}|=$1 &100 &500 &1000 \\
        \midrule
            Midjourney & Mid. v6 & 75.9 & 7590.0 &37950.0 &75900.0\\
            OpenAI  & DALL$\cdot$E3 & 15.5 & 1550.0 &7750.0 &15500.0\\
            Replicate & SD1.5 & 32.9 &3290.0 & 16450.0 &32900.0\\
            Together.AI & SD2.1 & 8.3 & 830.0  & 4150.0 & 8300.0\\
            Replicate & SD3 med. & \textbf{5.8} & 580.0 & 2900.0 & 5800.0\\
            Stability.AI & SD3 & 8.2 & 820.0 & 4100.0 & 8200.0\\
            Stability.AI & SD3.5 med. & 9.8 & 980.0 & 4900.0 & 9800.0\\
            Fal.AI & SD3.5 large & 7.9 & 790.0 & 3950.0 & 7900.0\\
            DEdgeAI & reSD3-m & 18.3 & \textbf{382.4} & \textbf{1921.5} & \textbf{3895.4}\\
        \bottomrule
    \end{tabular}}
    \label{table6}
\end{table}

\subsection{Further Discussion}
\textbf{Decision-making Overhead.} Diffusion models are inherently computationally intensive, with slow inference speeds in the field of image generation. We test the proposed LAD-TS algorithm with a decision-making time of 0.36 milliseconds. This is because the proposed LADN model is designed with only two hidden layers with 20 neurons each. Meanwhile, our LAD-TS algorithm is implemented with a linear time complexity (See in Theorem \ref{Theorem3}). These characteristics indicate that our LAD-TS method has the potential to handle scenarios with the requirement of low latency responses in edge systems.

\textbf{Security and Privacy.} Our current framework assumes trusted ESs and secure communication channels, as commonly done in prior works. However, the distributed edge system may lead to some issues about privacy and security, such as man-in-the-middle and model inversion attacks, inference leakage from latent representations or model outputs, and cross-edge communication vulnerabilities. Therefore, in future work, using secure multi-party computation, differential privacy, and federated learning are promising future directions for enhancing privacy protection and resilience against adversarial threats in multi-edge AIGC systems.

\textbf{Limitations and Scalability.} Our LAD-TS method presents the integration of the diffusion model and the SAC framework, improving the AIGC task scheduling policy. However, this paper focuses on diffusion-based AIGC workloads, for which the request-level computation demand can be characterized by the model-dependent per-step cost and the number of denoising steps. The developed DEdgeAI system prototype is therefore validated using the diffusion-based SD3-medium and reSD3-m models. Other architecture's AIGC models (e.g., LLM inference) exhibit substantially different execution characteristics, including separate prefill and decoding stages, variable output lengths, token-level dependencies, and dynamically growing KV-cache requirements. Directly applying the current workload model to such pipelines would therefore be inappropriate. Extending LAD-TS to other AIGC inference would require redesigning the workload representation, state space, memory constraints, and scheduling granularity, which is left as an important direction for future work.

\section{Conclusion}\label{sec7}
In this paper, we address the distributed AIGC challenge within collaborative edge environments by formulating it as an online INLP problem focused on minimizing task service delays for a real AIGC service. We propose a novel LAD-TS method for task scheduling by incorporating the diffusion model and DRL technique. Moreover, we design a latent action diffusion strategy to optimize task scheduling policy with the proof of probability derivation, improving the AIGC QoE and training time. We implement the LAD-TS as an online distributed algorithm that is characterized by linear time complexity. Furthermore, we develop the DEdgeAI system with our LAD-TS method and implement a refined AIGC service deployment. Extensive simulations underscore our method's superiority, achieving service delay reductions between $8.58\%$ and $30.67\%$ while closely approximating optimal outcomes. Also, LAD-TS significantly decreases the required training episodes by at least $60\%$. Empirical tests on the DEdgeAI system reveal that the LAD-TS reduces service delay by $7.36\%$ to $18.11\%$ compared to the state-of-the-art methods. The reSD3-m achieves a $60\%$ reduction in memory occupation over the original SD3 medium model. The DEdgeAI reduces the AIGC service delays by $34.07\%$ to $94.96\%$ over compared to existing representative platforms. 

\bibliographystyle{IEEEtran}  
\bibliography{references} 

@misc{dalle,
    howpublished = {DALL$\cdot$E Platform: \url{https://openai.com/index/dall-e/}}
}

@misc{huggingface,
    howpublished = {Hugging Face Platform: \url{https://huggingface.co/spaces}}    
}

@article{touvron2023llama,
  title={Llama: Open and efficient foundation language models},
  author={Touvron, Hugo and Lavril, Thibaut and Izacard, Gautier and Martinet, Xavier and Lachaux, Marie-Anne and Lacroix, Timoth{\'e}e and Rozi{\`e}re, Baptiste and Goyal, Naman and Hambro, Eric and Azhar, Faisal and others},
  journal={arXiv preprint arXiv:2302.13971},
  year={2023}
}

@misc{midjourney,
    howpublished = {Midjourney Platform: \url{https://www.midjourney.com/home}}    
}

@misc{sd3medium,
    howpublished = {Stable Diffusion 3-Medium: \url{https://github.com/Stability-AI/sd3-ref/tree/master}}    
}

@misc{stablityai,
    howpublished = {Stability.AI Platform: \url{https://stability.ai/}}    
}

@misc{togetherai,
    howpublished = {Together.AI Platform: \url{https://www.together.ai/}}
}

@misc{replicate,
    howpublished = {Replicate Platform: \url{https://replicate.com/}}   
}

@misc{falai,
    howpublished = {Fal.AI Platform: \url{https://fal.ai/}}   
}

@inproceedings{esser2024scaling,
  title={Scaling rectified flow transformers for high-resolution image synthesis},
  author={Esser, Patrick and Kulal, Sumith and Blattmann, Andreas and Entezari, Rahim and M{\"u}ller, Jonas and Saini, Harry and Levi, Yam and Lorenz, Dominik and Sauer, Axel and Boesel, Frederic and others},
  booktitle={Forty-first International Conference on Machine Learning},
  year={2024}
}

@article{cao2023comprehensive,
author={Cao, Y and Li, S and Liu, Y and Yan, Z and Dai, Y and Yu, P S and Sun, L},
title={A comprehensive survey of ai-generated content (aigc): A history of generative ai from gan to chatgpt},
journal={arXiv preprint arXiv:2303.04226},
pages={1--44},
year={2023},
publisher={IEEE}
}

@article{chu2023online,
  title={Online optimal service selection, resource allocation and task offloading for multi-access edge computing: A utility-based approach},
  author={Chu, Weibo and Yu, Peijie and Yu, Zhiwen and Lui, John CS and Lin, Yi},
  journal={IEEE Transactions on Mobile Computing},
  volume={22},
  number={7},
  pages={4150--4167},
  year={2023},
  publisher={IEEE}
}

@article{croitoru2023diffusion,
  title={Diffusion models in vision: A survey},
  author={Croitoru, Florinel-Alin and Hondru, Vlad and Ionescu, Radu Tudor and Shah, Mubarak},
  journal={IEEE Transactions on Pattern Analysis and Machine Intelligence},
  volume={45},
  number={9},
  pages={10850--10869},
  year={2023},
  publisher={IEEE}
}

@article{du2024exploring,
  author={Du, Hongyang and Zhang, Ruichen and Niyato, Dusit and Kang, Jiawen and Xiong, Zehui and Kim, Dong In and Shen, Xuemin and Poor, H. Vincent},
  journal={IEEE Network}, 
  title={Exploring Collaborative Distributed Diffusion-Based AI-Generated Content (AIGC) in Wireless Networks}, 
  year={2024},
  volume={38},
  number={3},
  pages={178-186}
}

@article{du2024enhancing,
  title={Enhancing Deep Reinforcement Learning: A Tutorial on Generative Diffusion Models in Network Optimization},
  author={Du, Hongyang and Zhang, Ruichen and Liu, Yinqiu and Wang, Jiacheng and Lin, Yijing and Li, Zonghang and Niyato, Dusit and Kang, Jiawen and Xiong, Zehui and Cui, Shuguang and others},
  journal={IEEE Communications Surveys \& Tutorials},
  year={2024},
  publisher={IEEE}
}

@article{du2023ai,
  title={AI-generated incentive mechanism and full-duplex semantic communications for information sharing},
  author={Du, Hongyang and Wang, Jiacheng and Niyato, Dusit and Kang, Jiawen and Xiong, Zehui and Kim, Dong In},
  journal={IEEE Journal on Selected Areas in Communications},
  year={2023},
  publisher={IEEE}
}

@article{du2024diffusion,
  title={Diffusion-based Reinforcement Learning for Edge-enabled AI-Generated Content Services},
  author={Du, Hongyang and Li, Zonghang and Niyato, Dusit and Kang, Jiawen and Xiong, Zehui and Huang, Huawei and Mao, Shiwen},
  journal={IEEE Transactions on Mobile Computing},
  year={2024},
  publisher={IEEE}
}

@article{fan2024collaborative,
  title={Collaborative service placement, task scheduling, and resource allocation for task offloading with edge-cloud cooperation},
  author={Fan, Wenhao and Zhao, Liang and Liu, Xun and Su, Yi and Li, Shenmeng and Wu, Fan and Liu, Yuan'an},
  journal={IEEE Transactions on Mobile Computing},
  volume={23},
  number={1},
  pages={238-256},
  year={2024},
  publisher={IEEE}
}

@article{farhadi2021service,
  title={Service placement and request scheduling for data-intensive applications in edge clouds},
  author={Farhadi, Vajiheh and Mehmeti, Fidan and He, Ting and La Porta, Thomas F and Khamfroush, Hana and Wang, Shiqiang and Chan, Kevin S and Poularakis, Konstantinos},
  journal={IEEE/ACM Transactions on Networking},
  volume={29},
  number={2},
  pages={779--792},
  year={2021},
  publisher={IEEE}
}

@article{guo2024betweenness,
  title={Betweenness approximation for edge computing with hypergraph neural networks},
  author={Guo, Yaguang and Xie, Wenxin and Wang, Qingren and Yan, Dengcheng and Zhang, Yiwen},
  journal={Tsinghua Science and Technology},
  volume={30},
  number={1},
  pages={331--344},
  year={2024},
  publisher={TUP}
}

@inproceedings{gu2021layer,
  title={Layer aware microservice placement and request scheduling at the edge},
  author={Gu, Lin and Zeng, Deze and Hu, Jie and Li, Bo and Jin, Hai},
  booktitle={IEEE INFOCOM 2021-IEEE Conference on Computer Communications},
  pages={1--9},
  year={2021},
  organization={IEEE}
}

@inproceedings{han2022edgetuner,
  title={EdgeTuner: Fast Scheduling Algorithm Tuning for Dynamic Edge-Cloud Workloads and Resources},
  author={Han, Rui and Wen, Shilin and Liu, Chi Harold and Yuan, Ye and Wang, Guoren and Chen, Lydia Y},
  booktitle={IEEE INFOCOM 2022-IEEE Conference on Computer Communications},
  pages={880--889},
  year={2022},
  organization={IEEE}
}

@InProceedings{haarnoja2018soft,
  title={Soft Actor-Critic: Off-Policy Maximum Entropy Deep Reinforcement Learning with a Stochastic Actor},
  author={Haarnoja, Tuomas and Zhou, Aurick and Abbeel, Pieter and Levine, Sergey},
  booktitle={Proceedings of the 35th International Conference on Machine Learning (PMLR)},
  pages={1861--1870},
  year={2018},
  volume={80},
  publisher={PMLR}
}

@inproceedings{he2024online,
  title={An Online Joint Optimization Approach for QoE Maximization in UAV-Enabled Mobile Edge Computing},
  author={He, Long and Sun, Geng and Sun, Zemin and Wang, Pengfei and Li, Jiahui and Liang, Shuang and Niyato, Dusit},
  booktitle={IEEE INFOCOM 2024-IEEE Conference on Computer Communications},
  pages={1--10},
  year={2024},
  organization={IEEE}
}

@article{ho2020denoising,
  title={Denoising diffusion probabilistic models},
  author={Ho, Jonathan and Jain, Ajay and Abbeel, Pieter},
  journal={Advances in Neural Information Processing Systems},
  volume={33},
  pages={6840--6851},
  year={2020}
}

@article{khochare2023improved,
  title={Improved algorithms for co-scheduling of edge analytics and routes for UAV fleet missions},
  author={Khochare, Aakash and Sorbelli, Francesco Betti and Simmhan, Yogesh and Das, Sajal K},
  journal={IEEE/ACM Transactions on Networking},
  volume={32},
  number={1},
  pages={17--33},
  year={2023},
  publisher={IEEE}
}

@article{liu2024sora,
  title={Sora: A Review on Background, Technology, Limitations, and Opportunities of Large Vision Models},
  author={Liu, Yixin and Zhang, Kai and Li, Yuan and Yan, Zhiling and Gao, Chujie and Chen, Ruoxi and Yuan, Zhengqing and Huang, Yue and Sun, Hanchi and Gao, Jianfeng and others},
  journal={arXiv preprint arXiv:2402.17177},
  year={2024}
}

@inproceedings{lin2014microsoft,
  title={Microsoft coco: Common objects in context},
  author={Lin, Tsung-Yi and Maire, Michael and Belongie, Serge and Hays, James and Perona, Pietro and Ramanan, Deva and Doll{\'a}r, Piotr and Zitnick, C Lawrence},
  booktitle={European conference on computer vision},
  pages={740--755},
  year={2014},
  organization={Springer}
}

@article{li2024online,
  title={Online optimization of dnn inference network utility in collaborative edge computing},
  author={Li, Rui and Ouyang, Tao and Zeng, Liekang and Liao, Guocheng and Zhou, Zhi and Chen, Xu},
  journal={IEEE/ACM Transactions on Networking},
  year={2024},
  publisher={IEEE}
}

@article{lin2024blockchain,
  title={Blockchain-based Efficient and Trustworthy AIGC Services in Metaverse},
  author={Lin, Yijing and Gao, Zhipeng and Du, Hongyang and Niyato, Dusit and Kang, Jiawen and Xiong, Zehui and Zheng, Zibin},
  journal={IEEE Transactions on Services Computing},
  year={2024},
  publisher={IEEE}
}

@article{luong2019applications,
  title={Applications of deep reinforcement learning in communications and networking: A survey},
  author={Luong, Nguyen Cong and Hoang, Dinh Thai and Gong, Shimin and Niyato, Dusit and Wang, Ping and Liang, Ying-Chang and Kim, Dong In},
  journal={IEEE communications surveys \& tutorials},
  volume={21},
  number={4},
  pages={3133--3174},
  year={2019},
  publisher={IEEE}
}

@article{maaz2023video,
  title={Video-chatgpt: Towards detailed video understanding via large vision and language models},
  author={Maaz, Muhammad and Rasheed, Hanoona and Khan, Salman and Khan, Fahad Shahbaz},
  journal={arXiv preprint arXiv:2306.05424},
  year={2023}
}

@article{mnih2015human,
  title={Human-level control through deep reinforcement learning},
  author={Mnih, Volodymyr and Kavukcuoglu, Koray and Silver, David and Rusu, Andrei A and Veness, Joel and Bellemare, Marc G and Graves, Alex and Riedmiller, Martin and Fidjeland, Andreas K and Ostrovski, Georg and others},
  journal={nature},
  volume={518},
  number={7540},
  pages={529--533},
  year={2015},
  publisher={Nature Publishing Group}
}

@article{ouyang2022training,
  title={Training language models to follow instructions with human feedback},
  author={Ouyang, Long and Wu, Jeffrey and Jiang, Xu and Almeida, Diogo and Wainwright, Carroll and Mishkin, Pamela and Zhang, Chong and Agarwal, Sandhini and Slama, Katarina and Ray, Alex and others},
  journal={Advances in Neural Information Processing Systems},
  volume={35},
  pages={27730--27744},
  year={2022}
}

@article{tang2023multi,
  title={Multi-user layer-aware online container migration in edge-assisted vehicular networks},
  author={Tang, Zhiqing and Mou, Fangyi and Lou, Jiong and Jia, Weijia and Wu, Yuan and Zhao, Wei},
  journal={IEEE/ACM Transactions on Networking},
  year={2023},
  publisher={IEEE}
}

@article{tang2022deep,
  title={Deep reinforcement learning for task offloading in mobile edge computing systems},
  author={Tang, Ming and Wong, Vincent WS},
  journal={IEEE Transactions on Mobile Computing},
  volume={21},
  number={6},
  pages={1985--1997},
  year={2022},
  publisher={IEEE}
}

@article{van2023chatgpt,
author={Van Dis, E. A. and Bollen, J. and Zuidema, W. and Van Rooij, R. and C. L. Bockting},
title={ChatGPT: five priorities for research},
journal={Nature},
volume={614},
number={7947},
pages={224--226},
year={2023}
}

@article{wang2023edge,
  title={Edge Computing and Sensor-Cloud: Overview, Solutions, and Directions},
  author={Wang, Tian and Liang, Yuzhu and Shen, Xuewei and Zheng, Xi and Mahmood, Adnan and Sheng, Quan Z},
  journal={ACM Computing Surveys},
  volume={55},
  number={13},
  pages={1--37},
  year={2023},
  publisher={ACM New York, NY}
}

@article{wang2024next,
  author={Wang, Shangshang and Shao, Ziyu and Lui, John C.S.},
  journal={IEEE Transactions on Mobile Computing}, 
  title={Next-Word Prediction: A Perspective of Energy-Aware Distributed Inference}, 
  year={2024},
  volume={23},
  number={5},
  pages={5695-5708}
}

@inproceedings{wang2024invar,
  title={INVAR: Inversion Aware Resource Provisioning and Workload Scheduling for Edge Computing},
  author={Wang, Bin and Irwin, David and Shenoy, Prashant and Towsley, Don},
  booktitle={IEEE INFOCOM 2024-IEEE Conference on Computer Communications},
  pages={1--10},
  year={2024},
  organization={IEEE}
}

@ARTICLE{wu2025qoe,
  author={Wu, Jiaqi and Zhuang, Xinyi and Tang, Ming and Gao, Lin},
  journal={IEEE Transactions on Mobile Computing}, 
  title={QoE-Aware Offloading and Resource Allocation for MEC-Empowered AIGC Services}, 
  year={2025},
  volume={24},
  number={10},
  pages={9664-9682}
  }

@ARTICLE{xu2025enhance,
  author={Xu, Changfu and Guo, Jianxiong and Zeng, Jiandian and Qiu, Houming and Wang, Tian and Chu, Xiaowen and Cao, Jiannong},
  journal={IEEE Transactions on Services Computing}, 
  title={Enhancing AIGC Service Efficiency With Adaptive Multi-Edge Collaboration in a Distributed System}, 
  year={2026},
  volume={19},
  number={1},
  pages={323-336}}

@article{xu2026eat,
  author={Xu, Zhifei and Tang, Zhiqing and Lou, Jiong and Yao, Zhi and Xie, Xuan and Wang, Tian and Wang, Yinglong and Jia, Weijia},
  journal={IEEE Transactions on Mobile Computing}, 
  title={EAT: QoS-Aware Edge-Collaborative AIGC Task Scheduling via Attention-Guided Diffusion Reinforcement Learning}, 
  year={2026},
  volume={25},
  number={7},
  pages={9980-9996}
}

@article{xu2024dynamic,
  title={Dynamic parallel multi-server selection and allocation in collaborative edge computing},
  author={Xu, Changfu and Guo, Jianxiong and Li, Yupeng and Zou, Haodong and Jia, Weijia and Wang, Tian},
  journal={IEEE Transactions on Mobile Computing},
  pages={1--15},
  year={2024},
  doi={10.1109/TMC.2024.3376550},
  organization={IEEE}
}

@article{xu2023sparks,
  author={Xu, Minrui and Niyato, Dusit and Zhang, Hongliang and Kang, Jiawen and Xiong, Zehui and Mao, Shiwen and Han, Zhu},
  journal={IEEE Vehicular Technology Magazine}, 
  title={Sparks of Generative Pretrained Transformers in Edge Intelligence for the Metaverse: Caching and Inference for Mobile Artificial Intelligence-Generated Content Services}, 
  year={2023},
  volume={18},
  number={4},
  pages={35-44},
  organization={IEEE}
}

@ARTICLE{yang2026online,
  author={Yang, Han and Zheng, Ying and Jiao, Lei and Xu, Yuedong and Li, Zongpeng},
  journal={IEEE Transactions on Networking}, 
  title={Online Request Scheduling for Quality-Aware Diffusion-Based AIGC Services}, 
  year={2026},
  volume={34},
  number={},
  pages={3196-3209}
}

@inproceedings{ye2024asteroid,
  title={Asteroid: Resource-Efficient Hybrid Pipeline Parallelism for Collaborative DNN Training on Heterogeneous Edge Devices},
  author={Ye, Shengyuan and Zeng, Liekang and Chu, Xiaowen and Xing, Guoliang and Chen, Xu},
  booktitle={Proceedings of the 30th Annual International Conference on Mobile Computing and Networking},
  pages={312--326},
  year={2024}
}

@article{zeng2023improved,
  title={Improved double deep Q network-based task scheduling algorithm in edge computing for makespan optimization},
  author={Zeng, Lei and Liu, Qi and Shen, Shigen and Liu, Xiaodong},
  journal={Tsinghua Science and Technology},
  volume={29},
  number={3},
  pages={806--817},
  year={2023},
  publisher={TUP}
}

@article{zhang2024incentive,
  title={Incentive mechanisms for online task offloading with privacy-preserving in UAV-assisted mobile edge computing},
  author={Zhang, Renli and Zhou, Ruiting and Wang, Yufeng and Tan, Haisheng and He, Kun},
  journal={IEEE/ACM Transactions on Networking},
  year={2024},
  publisher={IEEE}
}

\begin{IEEEbiography}
[{\includegraphics[width=1in,height=1.25in,clip,keepaspectratio]{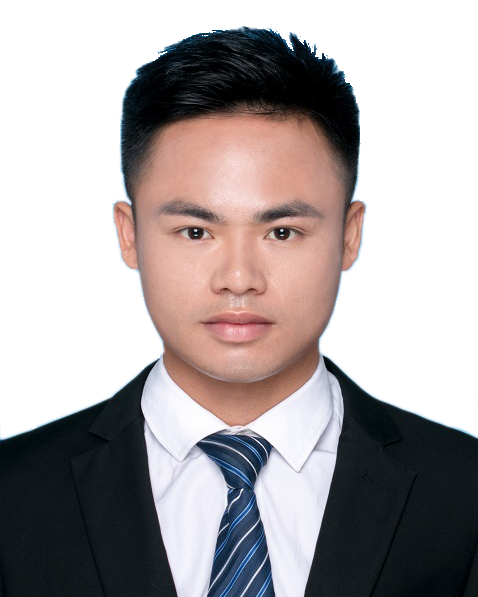}}]
{Changfu Xu}
(Member, IEEE) received the B.S. degree in Communication Engineering and the M.S. degree in Software Engineering from Jiangxi University of Finance and Economics in 2015 and 2018, respectively. He received his Ph.D. degree in Computer Science and Technology from Hong Kong Baptist University, Hong Kong, China, in 2025. Currently, he is an Assistant Professor with the School of Software and IoT Engineering, Jiangxi University of Finance and Economics, Nanchang, China. His main research interests include edge computing and AIGC. He has published more than 20 papers.
\end{IEEEbiography}

\begin{IEEEbiography}
[{\includegraphics[width=1in,height=1.25in,clip,keepaspectratio]{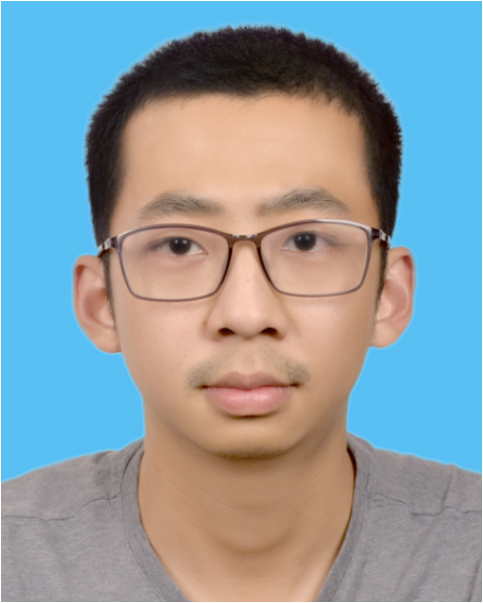}}]
{Jianxiong Guo}
(Member, IEEE) received his Ph.D. degree from the Department of Computer Science, University of Texas at Dallas, Richardson, TX, USA, in 2021, and his B.E. degree from the School of Chemistry and Chemical Engineering, South China University of Technology, Guangzhou, China, in 2015. He is currently an Associate Professor with the Advanced Institute of Natural Sciences, Beijing Normal University, and the Guangdong Key Lab of AI and Multi-Modal Data Processing, BNU-HKBU United International College, Zhuhai, China.  He is a member of IEEE/ACM/CCF. He has published more than 80 papers. His research interests include IoT applications and combinatorial optimization.
\end{IEEEbiography}

\begin{IEEEbiography}
[{\includegraphics[width=1in,height=1.25in,clip,keepaspectratio]{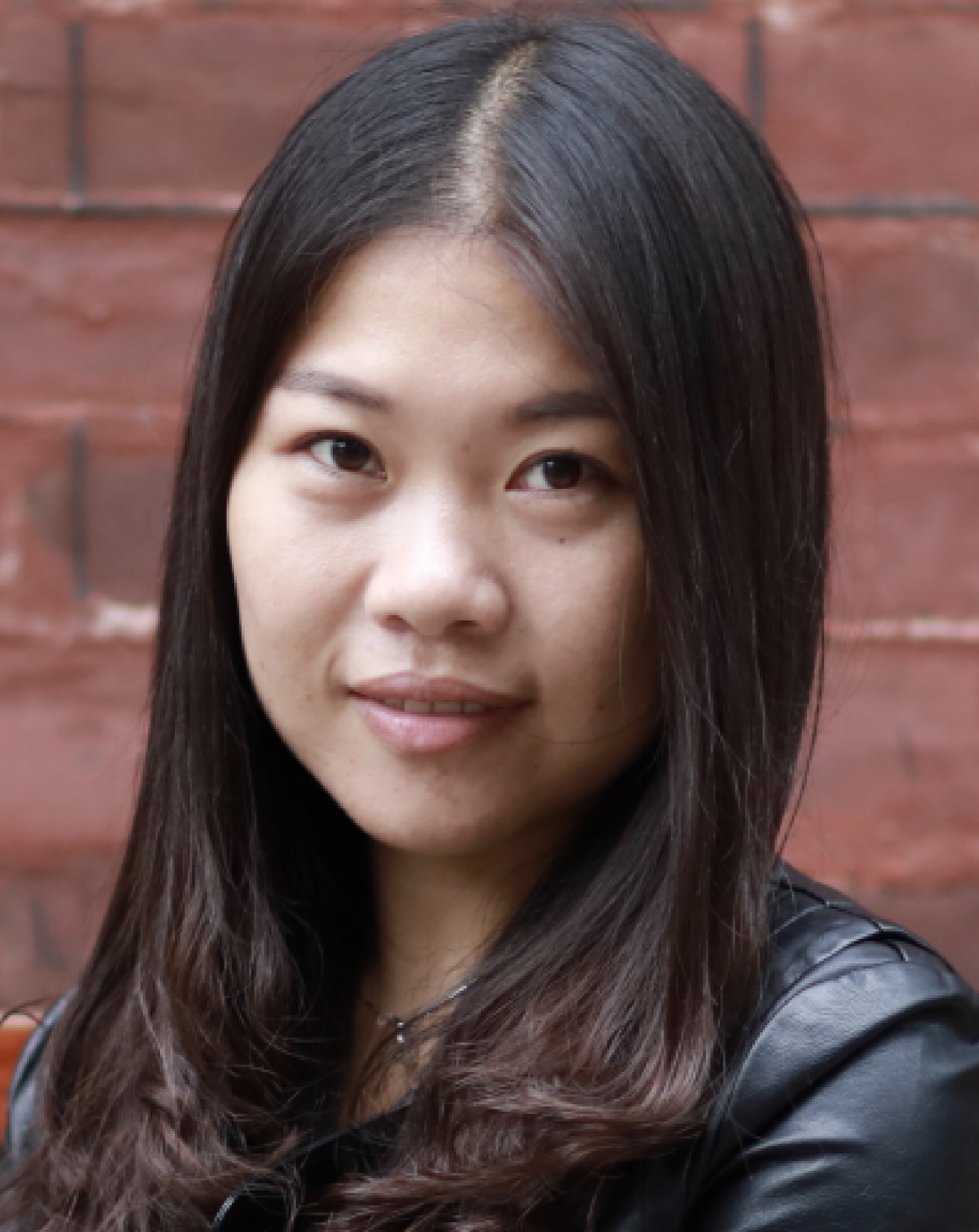}}]
{Wanyu Lin}
(Member, IEEE) received the B.Engr. degree from the School of Electronic Information and Communications, Huazhong University of Science and Technology, Wuhan, China, in 2012, the M.Phil. degree from the Department of Computing, The Hong Kong Polytechnic University, Hong Kong, in 2015, and the Ph.D. degree from the Department of Electrical and Computer Engineering, University of Toronto, Toronto, ON, Canada, in 2020. She is an Assistant Professor with the Department of Computing, The Hong Kong Polytechnic University. Her research interests include graph deep learning and collaborative learning.
\end{IEEEbiography}

\begin{IEEEbiography}
[{\includegraphics[width=1in,height=1.25in,clip,keepaspectratio]{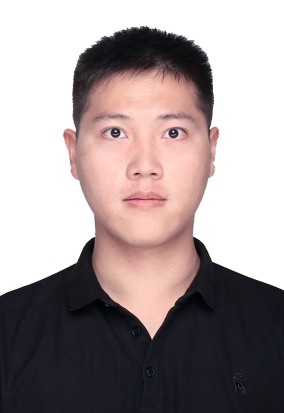}}]
{Haodong Zou}
(Member, IEEE) received the B.S. degree in Software Engineering and the M.S. degree in Computer Science and Technology from Anhui University, China, in 2018 and 2021, respectively. He received his Ph.D. degree in Computer Science and Technology from Hong Kong Baptist University, Hong Kong, China, in 2025. Currently, he is a postdoctoral fellow with the Anhui University, Hefei, China. His main research interests include 5G mobile communication, edge computing and edge intelligence.
\end{IEEEbiography}

\begin{IEEEbiography}
[{\includegraphics[width=1in,height=1.25in,clip,keepaspectratio]{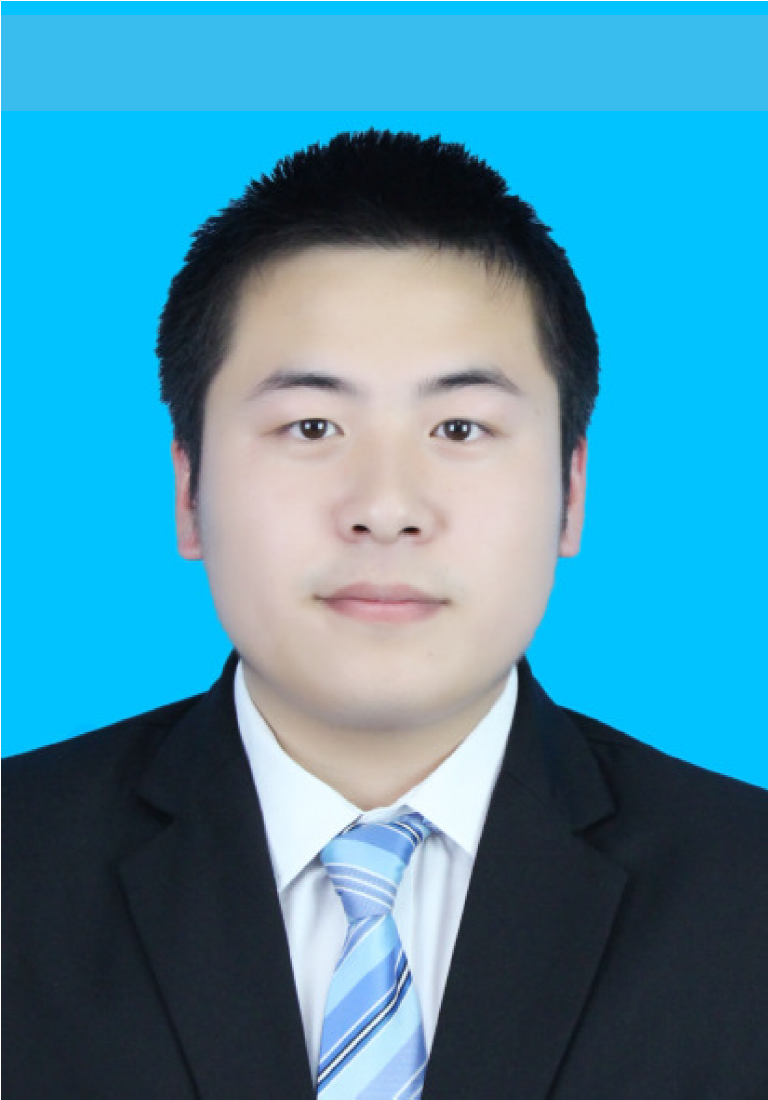}}]
{Yuzhu Liang}
(Member, IEEE) received his B.S degree and M.S. degree from Huaqiao University, Xiamen, China, in 2017 and 2020, respectively. He received his Ph.D. degree from Beijing Normal University, Beijing, China, in 2025. Currently, he is a postdoctoral fellow with the Beijing Normal University, Zhuhai, China. He has published more than 30 papers. His research interests include edge computing, mobile computing, and artificial intelligence.
\end{IEEEbiography}

\begin{IEEEbiography}
[{\includegraphics[width=1in,height=1.25in,clip,keepaspectratio]{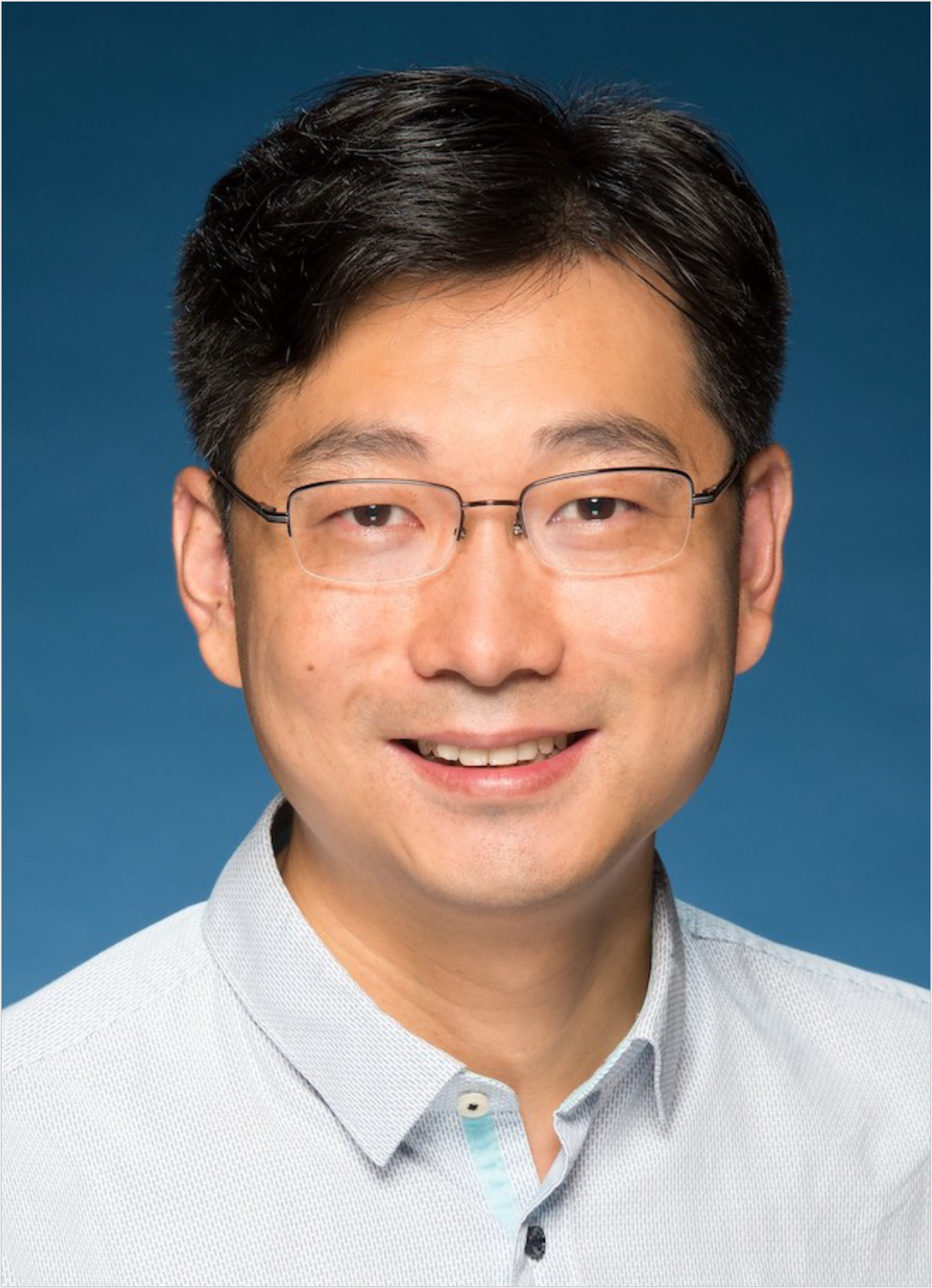}}]
{Xiaowen Chu}
(Fellow, IEEE) received the B.Eng. degree in computer science from Tsinghua University, Beijing, China, in 1999, and the Ph.D. degree in computer science from The Hong Kong University of Science and Technology, Hong Kong, in 2003. He is currently a Full Professor and Head of Data Science and Analytics Thrust at The Hong Kong University of Science and Technology (Guangzhou). His research interests include GPU Computing, Distributed Machine Learning, Cloud Computing, and Wireless Networks. He has published over 240 research articles at international journals and conference proceedings. He has won several Best Paper Awards at different international conferences, including IEEE INFOCOM 2021 and ACM EuroSys 2025. 
\end{IEEEbiography}

\begin{IEEEbiography}
[{\includegraphics[width=1in,height=1.25in,clip,keepaspectratio]{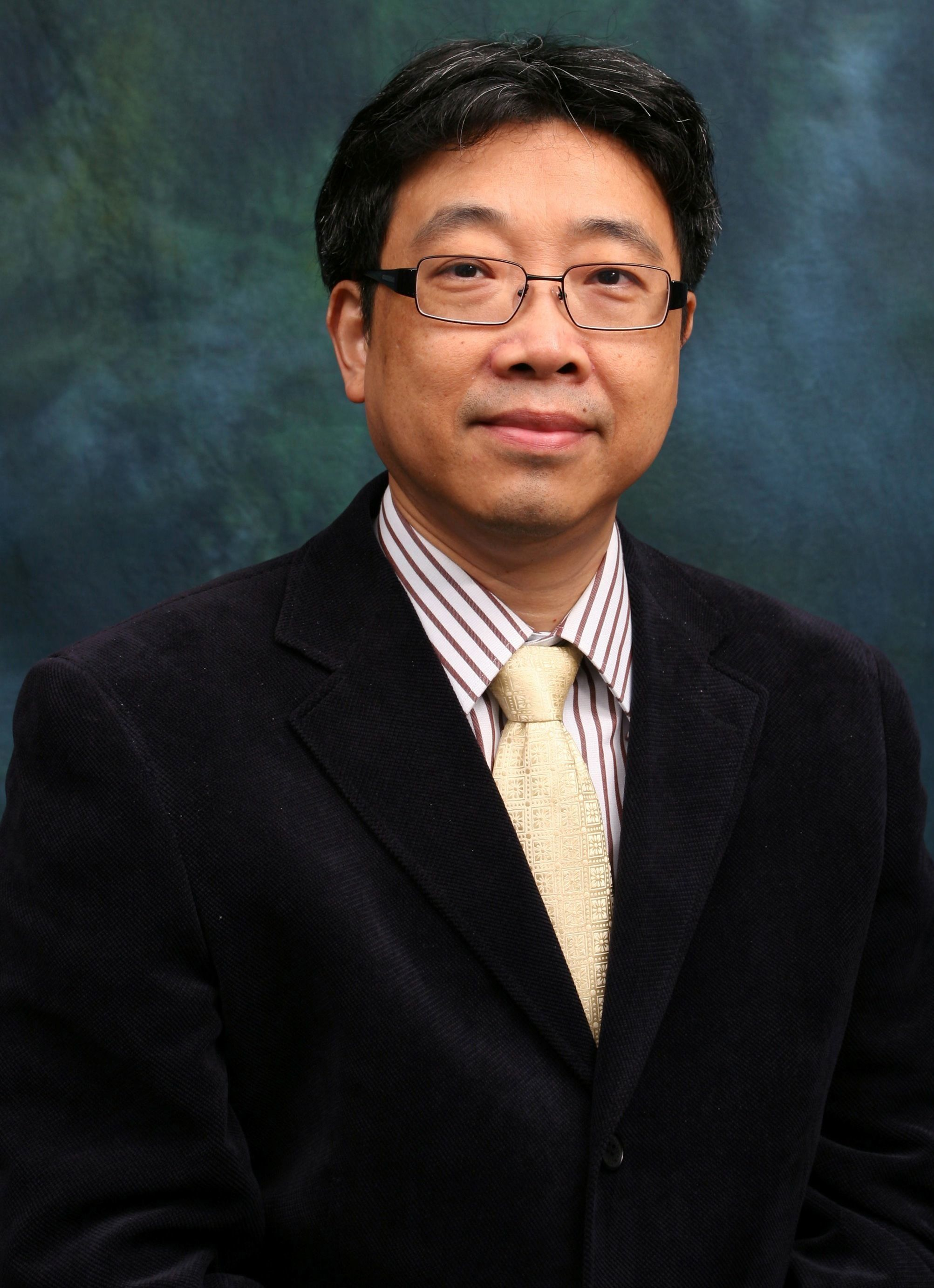}}]
{Jiannong Cao}
(Fellow, IEEE) received the B.Sc. degree in computer science from Nanjing University, China, in 1982, and the M.Sc. and Ph.D. degrees in computer science from Washington State University, USA, in 1986 and 1990, respectively. He is currently the Otto Poon Charitable Foundation Professor in data science and the Chair Professor of distributed and mobile computing with the Department of Computing, Hong Kong Polytechnic
University, Hong Kong. He is also the Director of the Internet and Mobile Computing Lab at the Department and the Associate Director of the University Research Facility in big data analytics. His research interests include parallel and distributed computing, wireless networks and mobile computing, big data and cloud computing, pervasive computing, and fault-tolerant computing. He has coauthored five books in Mobile Computing and Wireless Sensor Networks, co-edited nine books, and published over 600 papers in major international journals and conference proceedings. He is a Distinguished Member of ACM and a Senior Member of the China Computer Federation (CCF).
\end{IEEEbiography}

\begin{IEEEbiography}
[{\includegraphics[width=1in,height=1.25in,clip,keepaspectratio]{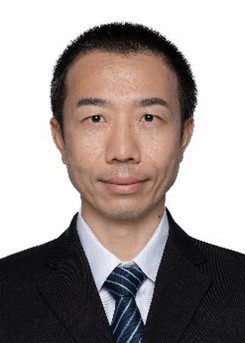}}]
{Tian Wang}
(Senior Member, IEEE) received his BSc and MSc degrees in Computer Science from Central South University in 2004 and 2007, respectively. He received his Ph.D. degree in Computer Science from the City University of Hong Kong in 2011. Currently, he is a professor at the Institute of Artificial Intelligence and Future Networks, Beijing Normal University. His research interests include Internet of Things, edge computing, and mobile computing. He has 30 patents and has published more than 200 papers in high-level journals and conferences. He was a co-recipient of the Best Paper Runner-up Award of IEEE/ACM IWQoS 2024. He has more than 20000 citations, according to Google Scholar. His H-index is 80.
\end{IEEEbiography}

\end{document}